\documentclass[mnsc]{informs3aa} 
\usepackage[final]{showkeys}
\usepackage{subfigure}
\usepackage{arydshln}
\usepackage{makecell}

\OneAndAHalfSpacedXI 

\usepackage{endnotes}

\usepackage{longtable}
\usepackage{tabularx}
\usepackage{booktabs}
\usepackage{etex}
\usepackage{multirow}
\usepackage{array}
\usepackage{bbm}

\usepackage[title]{appendix}

\newcommand{\vct}{\boldsymbol }

\newcommand{\op}{\mathrm{op}}
\newcommand{\kl}{\mathrm{KL}}

\renewcommand{\tilde}{\widetilde}
\renewcommand{\hat}{\widehat}
\renewcommand{\bar}{\overline}

\newcommand{\mad}{\mathrm{MAD}}

\newcommand{\diag}{\mathrm{diag}}

\newcommand{\poly}{\mathrm{poly}}

\newcommand{\tr}{\mbox{\rm tr}\, }

\renewcommand{\hat}{\widehat}
\renewcommand{\tilde}{\widetilde}

\usepackage{amsbsy, amsfonts, amsgen, amsmath, amsopn, amssymb, amstext,
amsxtra, bezier, color, enumerate, graphicx, latexsym, verbatim,
pictexwd, supertabular, url, dsfont,leftidx, mathrsfs, appendix, setspace}
\usepackage{algorithm}
\usepackage{algorithmicx}
\usepackage{algpseudocode}
\makeatletter
\def\BState{\State\hskip-\ALG@thistlm}
\makeatother

\usepackage{pifont}

\usepackage{natbib}
 \bibpunct[, ]{(}{)}{,}{a}{}{,}%
 \def\bibfont{\small}%
\usepackage[colorlinks=true,bookmarks=false,urlcolor=blue, citecolor=blue,linkcolor=blue,bookmarksopen=false,draft=false]{hyperref}

\newcommand{\ud}{\mathrm d}

\usepackage{parskip} 
		
\newcommand{\mA}{\mathcal A}
\newcommand{\mE}{\mathcal E}
\newcommand{\mX}{\mathcal X}
\newcommand{\mF}{\mathcal F}
\newcommand{\mZ}{\mathcal Z}
\newcommand{\mD}{\mathcal D}

\newcommand{\mO}{\mathcal O}
\newcommand{\mB}{\mathcal{B}}

\newcommand{\reg}{\mathrm{Reg}}
\newcommand{\lap}{\mathrm{Lap}}
\newcommand{\tot}{\mathrm{tot}}

\newcommand{\tv}{\mathrm{TV}}
		
\TheoremsNumberedThrough     
\ECRepeatTheorems

\EquationsNumberedThrough    

\MANUSCRIPTNO{} 

\begin{document}


\RUNAUTHOR{Li et al.}

\RUNTITLE{Optimal Locally Private Linear Contextual Bandit}

\TITLE{On the Optimal Regret of Locally Private Linear Contextual Bandit}

\ARTICLEAUTHORS{%
\AUTHOR{Jiachun Li}
\AFF{Institute for Data, Systems and Society, Massachusetts Institute of Technology, Cambridge, MA 02142, USA}
\AUTHOR{David Simchi-Levi}
\AFF{Institute for Data, Systems and Society, Operations Research Center, Department of Civil and Environmental Engineering,
Massachusetts Institute of Technology, Cambridge, MA 02142, USA}
\AUTHOR{Yining Wang}
\AFF{Naveen Jindal School of Management, University of Texas at Dallas, Richardson, TX 75080, USA}
} 

\ABSTRACT{Contextual bandit with linear reward functions is among one of the most extensively studied models in bandit and online learning research.
Recently, there has been increasing interest in designing \emph{locally private} linear contextual bandit algorithms, where sensitive information
contained in contexts and rewards is protected against leakage to the general public.
While the classical linear contextual bandit algorithm admits cumulative regret upper bounds of $\tilde O(\sqrt{T})$ via multiple alternative methods,
it has remained open whether such regret bounds are attainable in the presence of local privacy constraints, with the state-of-the-art result being $\tilde O(T^{3/4})$.
In this paper, we show that it is indeed possible to achieve an $\tilde O(\sqrt{T})$ regret upper bound for locally private linear contextual bandit.
Our solution relies on several new algorithmic and analytical ideas, such as the analysis of mean absolute deviation errors and layered principal component regression
in order to achieve small mean absolute deviation errors.

 This version: \today
}



\date{}

\maketitle

\section{Introduction}

Contextual bandit with linear reward functions is among one of the most extensively studied models in bandit and online learning research.
In this model, at every time period an algorithm receives a context and must select among one of the many actions. Given a pair of context and selected action,
the expected reward is modeled by a $d$-dimensional \emph{linear} model.
The objective of the algorithm is to balance estimation of reward models and exploitation of high-reward actions, so that the cumulative reward is maximized (or equivalently,
the cumulative regret is minimized). Sec.~\ref{subsec:model} provides a rigorous mathematical formulation of the linear contextual bandit model studied in this paper.
There has been many algorithms that are designed for the linear contextual bandit problem, such as the LinUCB and SupLinUCB algorithms designed specifically for linear or generalized linear models \citep{abbasi2011improved,rusmevichientong2010linearly,auer2002using,filippi2010parametric} and more general algorithms that apply to general function classes beyond the linear case \citep{auer2002nonstochastic,simchi2022bypassing,agarwal2014taming,foster2020beyond,xu2020upper}.
While there are subtle and nuanced differences in problem settings, assumptions and regret dependency, one common feature of these existing results is that all algorithms enjoy worst-case regret upper bound at the asymptotic order of $\tilde O(\sqrt{T})$
when there are $T$ time periods, which is minimax optimal for gap-free contextual bandit problems \citep{auer2002nonstochastic}.

In recent years, the issue of \emph{data privacy} \citep{dwork2014algorithmic} has received increasing interest in computer science and operations research communities, including in contextual bandit problems.
This is motivated by the observation that, in many real-world applications of contextual bandit, context vectors and actions contain sensitive information such as past purchases, credit worthiness and
offered prices, which should not be disclosed to the general public \citep{chen2023differential,chen2022privacy,fu2022privacy,lei2023privacy}. 
Using the notion of \emph{local privacy} first introduced by \cite{duchi2018minimax}, the work of \cite{zheng2020locally} designed algorithms for linear and generalized linear contextual bandits
with a regret upper bound of $\tilde O(T^{3/4})$.
While improved regret is possible with strong assumptions (especially \emph{eigenvalue-type conditions}, as in \citep{han2021generalized} which we also discuss in more details in Sec.~\ref{subsec:related-works}),
it is not known whether the regret could be improved to match $\tilde O(\sqrt{T})$ in the non-private case without strong eigenvalue conditions.
Indeed, it was suggested in \citep{zheng2020locally} that $\tilde O(\sqrt{T})$ regret may not be possible with local privacy constraints in the general case, and lower bounds should be pursued.

In this paper, we make two contributions towards the locally private linear contextual bandit problem:
\begin{enumerate}
\item On the lower bound side, we formalize the intuitions behind impossibility arguments in \citep{zheng2020locally}, by rigorously showing that 1) no analysis based on the \emph{mean-square error} could achieve a regret upper bound lower than $\tilde O(T^{3/4})$, and 2) no bandit algorithm based on \emph{input perturbation} could achieve a regret upper bound lower than $\tilde O(T^{2/3})$.
This two (partially) negative results explain why it is so challenging to achieve $\tilde O(\sqrt{T})$ regret for locally private linear contextual bandit, because many local privacy mechanisms
for linear models are input perturbation methods, and virtually all existing analysis of linear or even general contextual bandit relies on the mean-square error notion.
\item On the upper bound side, we show that despite the negative results, it is actually \emph{possible} to achieve $\tilde O(\sqrt{T})$ regret with local privacy constraints.
The algorithm and analysis that achieve this result would need to circumvent the negative results established: that they \emph{cannot} be based on the analysis of the mean-square error,
and they must use locally private mechanisms more sophisticated than input perturbation, which essentially rules out non-interactive local privacy mechanisms.
In our analysis, we rely on the \emph{mean-absolute deviation} (MAD) error and use a layered principal component regression approach to \emph{adaptively} partition contextual vectors
into hierarchical bins. Because of the non-standard algorithm and analysis, we believe the results of this paper would have implications to a wider class of privacy-constrained bandit problems as well.
Our results also leave a sequence of follow-up works open, as we remark in the final section of this paper.
\end{enumerate}

The rest of the paper is organized as follows. In the remainder of the introduction section we formally introduce the model formulation, including linear contextual bandit and locally private bandit models.
We also cover some additional related works on contextual bandit and privacy and position our contributions within the literature.
We then proceed to introduce two negative results and afterwards our main algorithm with regret upper bounds.
Finally, we complete the paper by mentioning several interesting future research directions.

\subsection{Model specification and definitions}\label{subsec:model}

We consider a standard linear contextual bandit model with finite number of actions and stochastic contexts. Let $\mathcal X$ be a compact context space and $\mathcal A$ be a discrete action space
such that $|\mathcal A|\leq A<\infty$. Let $P_{\mX}$ be an unknown distribution supported on $\mX$. Let $\mF$ be a hypothesis class such that each function $f\in\mF: \mX\times\mA\to\mathbb R$
maps a pair of context and action into an expected reward. In this paper, we focus exclusively on the \emph{linear} hypothesis class, for which every $f\in\mF$ can be parameterized as
$$
f(x,a) = \langle\phi(x,a),\theta_f\rangle,
$$
where $\theta_f\in\mathbb R^d$, $\|\theta_f\|_2\leq 1$ is a linear model associated with the hypothesis $f\in\mF$ and $\phi: \mX\times\mA\to\{z\in\mathbb R^d: \|z\|_2\leq 1\}$
is a known feature map that maps a context-action pair to a $d$-dimensional vector with bounded norm.

The linear contextual bandit problem consists of $T$ consecutive time periods equipped with unknown context distribution $P_{\mX}$ and an unknown linear model $f^*\in\mF$. 
At the beginning of time period $t\in[T]$, a context $x_t\sim P_{\mX}$ is sampled\footnote{For some contextual bandit algorithms such as LinUCB and EXP4, the contexts do not have to be stochastic
and could even be adaptively adversarial. Our analysis unfortunately does not work for such adversarial context settings. See Sec.~\ref{sec:conclusion} for further discussion on this point.} from $P_{\mX}$ and revealed to the
bandit algorithm. The algorithm then must take an action $a_t\in\mA$. Afterwards, a random reward $y_t\in[-1,1]$ is realized and observed such that
$$
\mathbb E[y_t|x_t,a_t] = f^*(x_t,a_t).
$$
The effectiveness of a policy $\pi$ is measured by its \emph{cumulative regret} against a clairvoyant policy $\pi^*$ which has perfect information of $P_{\mX}$ and $f^*$. 
More specifically, the cumulative regret of policy $\pi$ is defined as
\begin{equation}
\reg(\pi) := \mathbb E^\pi\left[\sum_{t=1}^T f^*(x_t,a_t^*) - f_t(x_t,a_t)\right]\;\;\;\;\;\;\text{where}\;\; a_t^* = \arg\max_{a\in\mA}f^*(x_t,a).
\label{eq:defn-reg}
\end{equation}
It is clear from definition that $\reg(\pi)$ is always non-negative, and smaller regret indicates better performance of a bandit algorithm.

A bandit algorithm or policy $\pi$ cannot make the action decision $a_t$ at time $t$ arbitrarily. In particular, the following two considerations are important:
\begin{enumerate}
\item \emph{admissibility}: the policy $\pi$ with the actual context distribution $P_{\mX}$ and model $f^*$ unknown, must gather data and information sequentially and make decisions
only based on \emph{past} data and observations;
\item \emph{privacy}: the policy $\pi$, while utilizing past data and observations to learn about context distributions or the underlying reward model $f^*$, must respect the \emph{privacy}
of past data involved in such procedures, including both context vectors and realized (observed) rewards.
\end{enumerate}

To formalize the two above-mentioned considerations, we use the following definition to rigorously define the class of bandit algorithms/policies allowed in this paper.
\begin{definition}
Let $\varepsilon\in(0,1]$ be a privacy parameter. A policy $\pi$ is \emph{admissible} and satisfies $\varepsilon$-local differential privacy ($\varepsilon$-LDP) if and only if it can be
parameterized as $\pi=(\mathfrak A_1,\mathfrak Z_1,\cdots,\mathfrak A_T,\mathfrak Z_t)$, such that for every time period $t$, the action $a_t\in\mA$ and a suitably defined internal anonymized
data entry $z_t\in\mZ$ are generated as:
\begin{enumerate}
\item $a_t\sim\mathfrak A_t(\cdot|z_1,\cdots,z_{t-1},x_t)$;
\item $z_t\sim\mathfrak Z_t(\cdot|z_1,\cdots,z_{t-1},x_t,a_t,y_t)$;
\end{enumerate}
furthermore, for every $x,x'\in\mX$, $a,a'\in\mA$, $y,y'\in[-1,1]$, $z_1,\cdots,z_{t-1}\in\mZ$ and every measurable $Z\subseteq \mZ$, it holds that
$$
\frac{\mathfrak Z_t(Z|z_1,\cdots,z_{t-1},x,a,y)}{\mathfrak Z_t(Z|z_1,\cdots,z_{t-1},x',a',y')} \leq e^{\varepsilon}.
$$
\label{defn:admissible}
\end{definition}

Definition \ref{defn:admissible} is similar to the definition of \emph{adaptive} local privacy in \citep{duchi2018minimax}, which is adapted to bandit problems in a ``non-anticipating''
sense standard to differentially private bandit formulations \citep{chen2023differential,chen2022privacy,shariff2018differentially}.
More specifically, when making the decision of $a_t$ the exact context $x_t$ could be utilized, and anonymization of $x_t$ (and $y_t$) only occurs \emph{after} time period $t$.
This is necessary because otherwise a linear regret lower bound applies, as shown in \citep{shariff2018differentially}.

\subsection{Related works}\label{subsec:related-works}

This works studies locally private linear contextual bandit, two very important questions in theoretical computer science, machine learning and operations research.
For contextual bandit, some representative references include \citep{auer2002nonstochastic,foster2020beyond,agarwal2014taming,simchi2022bypassing,beygelzimer2011contextual,xu2020upper}.
When specialized to linear or generalized linear settings, \emph{optimism} is an alternative effective approach which in general could deal with adversarial contexts and potentially very large action
space \citep{auer2002using,rusmevichientong2010linearly,abbasi2011improved,filippi2010parametric}.
The works of \cite{bubeck2012regret,lattimore2020bandit} provide excellent surveys on results of contextual bandit research.

The notion of \emph{differential privacy} was formalized in the seminar paper of \cite{dwork2006differential} and has seen a great amount of research since then.
We refer the readers to \citep{dwork2014algorithmic} for an excellent survey and review of progress in differential privacy research.
\emph{Local differential privacy} was first proposed by \cite{duchi2018minimax} as a strengthening to centralized differential privacy that is more applicable
to learning problems and statistical analysis, and has since then become a popular differential privacy measure for a number of statistics and machine learning problems.
The major difference of local differential privacy as opposed to the classical centralized privacy measure is the notion of an untrusted data aggregator, so that individual data samples must be
anonymized \emph{prior to} any statistical or algorithmic procedures, which is in contrast to centralized privacy that only needs to anonymize any publicized outputs.
The works of \citep{duchi2018minimax,duchi2024right} studied optimal rates of convergence for a number of models subject to local privacy constraints.
\cite{wang2023generalized} studied sample complexity upper bounds for generalized linear model using the criterion of \emph{excess risk}, which is different from
model estimating errors required in the bandit setting.
The work of \cite{chen2023differential} also provides an overview and comparison between local and centralized privacy notions under bandit and pricing settings.

We next discuss in further details of existing works on differentially private contextual bandit research.
The work of \citep{shariff2018differentially} is among one of the first papers studying differentially private linear bandit, by proposing the anticipating differential privacy notion
and analyzing a LinUCB procedure with perturbed sufficient statistics. The paper focuses solely on the centralized privacy setting and achieves a regret upper bound of $\tilde O(\sqrt{T}/\alpha)$
omitting other problem dependent parameters, where $\alpha$ is the privacy parameter.
The algorithm and analysis could be easily adapted to local privacy settings achieving a regret upper bound of $\tilde O(T^{3/4}/\alpha)$.
Going beyond linear models, \cite{chen2022privacy} extends the methods to generalized linear pricing models under centralized privacy settings,
requiring more sophisticated methods such as perturbed MLE and tree-based covariance anonymizations.
\cite{lei2023privacy} studies a similar model under offline settings and both centralized and local privacy notions.
\cite{chen2023differential} studies local privacy algorithms for \emph{non-parametric} contextual pricing models, built upon the models introduced by \cite{chen2021nonparametric}.
The works of \cite{zheng2020locally,han2021generalized} studied local privacy algorithms for generalized linear models.

The work of \cite{han2021generalized} achieves $\tilde O(\sqrt{T}/\alpha)$ regret for locally private generalized linear bandit, under the additional assumption that the context vectors
are relatively ``spread-out'' so that the least eigenvalue of the sample covariance or the projection of the context vectors is not too small.
Such eigenvalue type conditions potentially make the problem much easier.
For example, for purely linear models if the sample covariance $\Lambda$ after $n$ time periods satisfies $\Lambda\succeq\Omega(1)\cdot I$,
then it is not difficult to prove that the anonymized sample covariance $\hat\Lambda$ satisfies \emph{spectral} approximation $\hat\Lambda-\Lambda\preceq \tilde O_P(1/\sqrt{n})\cdot \Lambda$ with high probability, 
which then translates to an $\tilde O(\sqrt{T}/\alpha)$ regret.
When $\Lambda$ is not well-conditioned (e.g.~$\lambda_{\min}(\Lambda)$ growing significantly slower than $O(n)$ while $\lambda_{\max}(\Lambda)$ grows linearly as $\Omega(n)$), however, such spectral approximation may not hold and the best one can obtain from such a procedure is an $\tilde O(T^{3/4}/\alpha)$ regret.
For generalized linear models, while algorithms are different (as we no longer have simple sufficient statistics), similar phenomenon is also observed by \cite{chen2022privacy}
which shows that with additional eigenvalue conditions the regret of a centralized private algorithm could be vastly improved.

A closely related bandit question to what we studied is the so-called \emph{linear bandit}, for each time period a bandit algorithm selects an action $a_t$
from a \emph{fixed} action space $\mA$ each time without observing contexts and the rewards given action $a_t\in\mA$ is modeled by a linear function.
While similar in appearance, such settings are fundamentally different from linear \emph{contextual} bandit because there is no context and therefore actions are usually \emph{not} considered
sensitive since they are selected solely by the algorithm.
\cite{hanna2022differentially} studied local and central differential privacy under this setting and obtain near-optimal regret using variations to experimental design type algorithms.
The work of \cite{hanna2023contexts} shows that linear contextual bandit can be reduced to a sequence of linear bandit problems without contexts.
However, such reduction uses all historical data and is less likely to be amenable to local privacy constraints.

\section{Motivating negative results and examples}\label{sec:negative}

Before presenting our algorithm and analysis, we discuss several motivating negative examples and results which shed light on why the locally private contextual bandit 
problem is technically challenging. For simplicity, in this section all problem setups are offline: they are sufficient to demonstrate that certain important steps in existing
contextual bandit analysis cannot be adapted under local privacy settings, or popular local private estimation techniques cannot be directly applied to deliver tight bounds
for contextual bandit problems.

\subsection{Fundamental limits of the mean-square error}

Most existing works on (linear) contextual bandit rely on estimating the underlying model with small mean-square errors.
More specifically, let $\mD$ be an unknown distribution supported on a compact subset of $\mathbb R^d$ and $\phi_1,\cdots,\phi_n\overset{i.i.d.}{\sim}\mD$.
Let $y_1,\cdots,y_n\in\mathbb R$ be noisy labels such that $\mathbb E[y_i|\phi_i]=\phi_i^\top\theta^*$. It is possible to obtain an estimate $\hat\theta_n$ based on $\{(x_i,y_i)\}_{i=1}^n$
such that with high probability,
\begin{equation}
\mathbb E_{\phi\sim\mD}\left[\big|\phi^\top(\hat\theta_n-\theta^*)\big|^2\right] = (\hat\theta_n-\theta^*)^\top \Lambda_\mD (\hat\theta_n-\theta^*) = \tilde O(n^{-1}),
\label{eq:ols-mse}
\end{equation}
where $\Lambda_{\mD}=\mathbb E_{\mD}[\phi\phi^\top]$ and in the $\tilde O(\cdot)$ notation we omit polynomial dependency on $d$ and $\ln n$.
Without privacy considerations, Eq.~(\ref{eq:ols-mse}) could be attained by simply using the ordinary least-squares, or its penalized version.
The work of \cite{simchi2022bypassing} (Bypassing-the-Monster) and \cite{foster2018practical} (RegCB) specifically rely on such oracles with $\tilde O(n^{-1})$ convergence
rate in order to obtain an $\tilde O(\sqrt{T})$ cumulative regret.
For the more classical LinUCB \citep{abbasi2011improved,rusmevichientong2010linearly}, SupLinUCB \citep{auer2002using} or high-dimensional regression algorithms \citep{chen2022nearly,bastani2020online}
 designed specifically for linear models,
the guarantee of Eq.~(\ref{eq:ols-mse}) is also essential in the regret analysis. All of these existing works on (linear) contextual bandit share the following two important properties:
\begin{enumerate}
\item Eq.~(\ref{eq:ols-mse}) must be established for \emph{any} distribution $\mD$ supported on (a compact subset of) $\mathbb R^d$, because the policy assignments of actions
will alter the distributions of $\phi$ which are essentially out of control of the bandit algorithm;
\item The right-hand side of Eq.~(\ref{eq:ols-mse}) must be on the order of $\tilde O(n^{-1})$, so that after certain Cauchy-Schwarz relaxations the resulting cumulative regret is upper bounded
by $\tilde O(\sqrt{T})$, omitting polynomial dependency on $d$ and $\ln T$ terms.
\end{enumerate}

In the following result, we show, quite surprisingly, that Eq.~(\ref{eq:ols-mse}) is impossible for \emph{any} locally private estimators on general $\mD$ distributions.
\begin{theorem}
Fix $\alpha\in(0,1]$ and $d=2$. There exists a numerical constant $\underline C_1>0$ such that the following holds: for any $n\geq 1$, there exists a distribution $\mD_n$ supported on
$\{x\in\mathbb R^d: \|x\|_2\leq 1\}$ and parameter class $\Theta_n\subseteq\{\theta\in\mathbb R^d:\|\theta\|_2\leq 1\}$ such that for any estimate $\hat\theta_n$ that is $\alpha$-locally private,
$$
\sup_{\theta^*\in\Theta_n}\mathbb E_{\theta^*,\mD_n}\left[(\hat\theta_n-\theta^*)^\top\Lambda_{\mD_n}(\hat\theta_n-\theta^*)\right]\geq \frac{\underline C_1}{(e^\alpha-1)\sqrt{n}},
$$
where $\Lambda_{\mD_n}=\mathbb E_{\mD_n}[\phi\phi^\top]$.
\label{thm:ols-mse}
\end{theorem}

Theorem \ref{thm:ols-mse} is proved using the information-theoretical tools developed in \citep{duchi2018minimax} and is given in the supplementary material.
Its basic idea is simple: we design the distribution $\mD$ in such a way that one particular direction is explored with only probability $1/\sqrt{n}$. More specifically,\footnote{In the actual proof the $1/\sqrt{n}$ term is decorated with some constants to make probabilistic arguments work. Details in the supplementary material.}
\begin{equation}
\phi\sim\mD_n: \;\;\;\;\;\; \phi = \left\{\begin{array}{ll} e_1&\text{ with probability $1-1/\sqrt{n}$},\\ e_2 & \text{ with probability $1/\sqrt{n}$}.\end{array}\right.
\label{eq:md-1}
\end{equation}
With local privacy constraints, it can be proved rigorously that no procedure can detect such exploration with high probability (think about counting the occurrences of a particular event:
if the event only occurs $O(\sqrt{n})$ times out of $n$ total periods then no locally private method can recover the count accurately).
On the other hand, because the direction is explored with probability at least $1/\sqrt{n}$, failure to estimate along the direction leads to an MSE of at least $\Omega(1/\sqrt{n})$, completing the proof.

Theorem \ref{thm:ols-mse} shows that for general distribution $\mD$, no locally private procedure could attain a mean-square error at the rate of $\tilde O(1/n)$.
Instead, the best one can do is $\tilde O(1/\sqrt{n})$ for the mean-square error. This combined with Cauchy-Schwarz inequalities in existing analysis (e.g.~Bypassing the monster \citep{simchi2022bypassing} or LinUCB \citep{abbasi2011improved,rusmevichientong2010linearly}) leads to a regret of $\tilde O(T^{3/4})$, which is sub-optimal.
Coincidentally, the regret scaling of $\tilde O(T^{3/4})$ is the same with bounds obtained in \citep{zheng2020locally}, which is unsurprising in light of Theorem \ref{thm:ols-mse} because the 
approach and analysis in \citep{zheng2020locally} are founded on mean-square errors, which must suffer from such fundamental limitations.

\subsection{Fundamental limits of input perturbation}

It is tempting to consider the result in Theorem \ref{thm:ols-mse} as an artifact of the mean-square error because the lower bound is constructed for a particular type of distribution
that explores some directions with very small probability. For such distributions, one might conjecture that the mean-absolute deviation
$$
\mathbb E_{\phi\sim\mD}\left[\big|\phi^\top(\hat\theta_n-\theta^*)\big|\right]
$$
is a more proper error measure. By Jensen's inequality if the MSE is $\tilde O(1/n)$ then the above error measure is upper bounded by $\tilde O(1/\sqrt{n})$ but the inverse is not true:
in the constructed counter-examples in Theorem \ref{thm:ols-mse} both error measures evaluate to $\tilde O(1/\sqrt{n})$.
It might be the case that we may still use existing contextual bandit algorithms with simple local privacy procedures such as input perturbation, only changing the analysis that circumvents the MSE
in a smart way and analyzing the mean absolute deviation instead.

Our following theorem shows that this is not the case. In particular, if one uses input perturbation to satisfy local privacy constraints, then no estimator could perform well even with 
the mean absolute deviation measure.

\begin{theorem}
Fix $\alpha\in(0,1]$ and $d=2$. There exists a numerical constant $\underline C_2>0$ such that the following holds: for any $n\geq 1$, there exists a distribution $\mD_n$ supported on 
$\{x\in\mathbb R^d:\|x\|_2\leq 1\}$ and parameter class $\Theta_n\subseteq\{x\in\mathbb R^d:\|x\|_2\leq 1\}$ such that for any estimator $\hat\theta_n$ obtained with input $\{\tilde\phi_i,\tilde y_i\}_{i=1}^n$,
where $\tilde\phi_i=\phi_i+\lap_d(0,\alpha^{-1})$, $\tilde y_i=y_i+\lap_1(0,\alpha^{-1})$, $\phi_i\overset{i.i.d.}{\sim}\mD_n$, $y_i=\phi_i^\top\theta^*$ and $\lap_d(\mu,\sigma)$ is the distribution of a $d$-dimensional vector with i.i.d.~components
that are distributed as a Laplace distribution with mean $\mu$ and scale parameter $\sigma$,
$$
\sup_{\theta^*\in\Theta_n}\mathbb E_{\mD_n,\theta^*}\left[\mE_{\mad}(\hat\theta_n,\theta^*;\mD_n)\right] \geq \frac{\underline C_2}{\sqrt[3]{n}},
$$
where $\mE_{\mad}(\hat\theta,\theta^*;\mD) = \mathbb E_{\phi\sim\mD}[|\phi^\top(\hat\theta-\theta^*)|]$ is the mean absolute deviation error of $\hat\theta$ with respect to $\mD$.
\label{thm:ols-mad}
\end{theorem}

Before commenting how Theorem \ref{thm:ols-mad} is proved, we first comment on its implications. In the setup of this theorem there is a noiseless linear regression problem $y_i=\phi_i^\top\theta^*$
but both the feature vectors $\{\phi_i\}_{i=1}^n$ and the labels $\{y_i\}_{i=1}^n$ are anonymized using Laplace mechanism into $\{\tilde\phi_i,\tilde y_i\}_{i=1}^n$, a technique known as
\emph{input perturbation} in differential privacy research \citep{dwork2014algorithmic}. Afterwards any estimator based on $\{\tilde\phi_i,\tilde y_i\}_{i=1}^n$ is automatically differentially private
due to closeness to post-processing.
Theorem \ref{thm:ols-mad} essentially shows that such a straightforward approach of input perturbation cannot succeed, by demonstrating that any estimator based on the anonymized inputs $\{\tilde\phi_i,\tilde y_i\}_{i=1}^n$ must have a mean absolute deviation error of $\Omega(n^{-1/3})$, falling short of the desired $\tilde O(n^{-1/2})$ rate.
As a result, more sophisticated local privacy methods beyond input perturbation must be designed to succeed.

We remark that many existing approaches to locally private contextual bandit are based on or could be easily re-formulated as input perturbation.
For example, for purely linear models a common approach is to maintain the sufficient statistics $\Lambda=\sum_{i=1}^n \phi_i\phi_i^\top$ and $\mu =\sum_{i=1}^ny_i\phi_i$,
which can be done using Wishart and Laplace mechanisms. Such an approach is \emph{not} sufficient to obtain $\sqrt{T}$-regret results, as implied by Theorem \ref{thm:ols-mad}.

The complete proof of Theorem \ref{thm:ols-mad} is given in the supplementary material. At a high level, the proof is based on the following idea: the design distribution $\mD_n$ is constructed in the following manner:\footnote{In the actual proof the $n^{-1/3}$ terms are decorated with some constants to make probabilistic arguments work. Details in the supplementary material.}
$$
\phi\sim\mD_n:\;\;\;\;\;\;\phi = \left\{\begin{array}{ll} (1/2, n^{-1/3})&\text{ with probability $1/2$},\\ (1/2, -n^{-1/3})& \text{ with probability $1/2$}.\end{array}\right.
$$
Then the distribution of the second component of the anonymized $\tilde\phi$ is a \emph{mixture} of Laplace distributions with symmetric means.
Because such mixture distribution is exactly degenerate, there is no statistical methods that could actually estimate the mixture means,
similar to the phenomenon observed in \citep{chen1995optimal} for Gaussian mixture models and the more recent works of \cite{ho2019singularity,ho2016convergece}
for more general finite mixture models. The actual proof is based on the analysis of $\chi^2$-divergence between degenerate mixture models with different number of mixtures,
which explicitly demonstrates the singularity exhibited in the mixture distributions. 

\subsection{Discussion: necessity of input partitioning}

We further deepen the discussion of how to obtain $\tilde O(\sqrt{T})$ regret for locally private linear contextual bandit, in light of the above negative examples.
Let $\delta\asymp 1/\sqrt{n}$ be a small positive number and consider the following two scenarios for $d=2$:
\begin{align}
\text{Case I}:\;\;\phi=\left\{\begin{array}{ll}e_1,\text{ with prob.~$1-\delta$},\\
e_2,\text{ with prob.~$\delta$};\end{array}\right.\quad\text{ vs. }\quad
\text{Case II}:\;\;\phi=\left\{\begin{array}{ll}(1,\sqrt{\delta}),\text{ with prob.~$0.5$},\\
(1,-\sqrt{\delta}),\text{ with prob.~$0.5$}.\end{array}\right.
\end{align}
The two scenarios are similar to the two adversarial examples for Theorems \ref{thm:ols-mse} and \ref{thm:ols-mad}, respectively, and have almost the same covariance matrix
of $\diag(1, \delta)$. However, a locally private algorithm must treat these two scenarios very differently:
\begin{enumerate}
\item In case I, it is not necessary and indeed not possible to do anything on the second direction, because that direction only appears for $O(\sqrt{n})$ times out of $n$ time periods
which cannot be detected with local privacy constraints;
\item In case II, it is mandatory that the algorithm estimate the linear model along the second direction, because otherwise the regret could be $O(n^{3/4})$.
Such estimation cannot be done using simple input perturbation, as demonstrated by Theorem \ref{thm:ols-mad}. However, if we know that only small entries would appear on the second direction,
we could elect to privatize $(\phi_1,\phi_2/\sqrt{\delta})$ which still has constant sensitivity levels, thereby significantly reducing the variance along the second direction.
\end{enumerate}

The above observations motivate us to consider ``input partitioning'' strategies, which partitions the input vectors into several different bins and calibrate heteroskedastic noises into them.
This also helps eliminate infrequent regions that cannot be estimated (so that they will not interfere with model estimations at other regions).
This principle of partitioning of input vectors will be carried out through out our algorithm design in the remainder of this paper.

\section{Action elimination framework}

To give a high-level description of our proposed policy and also isolate the particularly challenging part of locally private linear contextual bandit,
we first describe and analyze an active elimination framework that assumes the access to a locally private offline regression oracle with uniform confidence intervals. 
The algorithm and analysis in this section are relatively simple and inspired heavily by the \texttt{RegCB.Elimination} algorithm in \citep{foster2018practical},
with two important differences. First, in our setting point-wise confidence intervals could be constructed in a uniform manner, which greatly simplifies the algorithm
and attains optimal regret bounds. Second, in our setting only the mean absolute deviation (MAD) error could be used, as discussed in the previous section.
As a result, certain steps of the analysis needs to be revised.

\subsection{Locally private offline regression oracle with confidence intervals}\label{subsec:oracle}

In this section we describe the definition of a locally private online regression oracle with confidence intervals,
which is then used in Algorithm \ref{alg:regcb} to construct a contextual bandit algorithm.
For notational simplicity, we augment $\mX$ to include a ``dummy'' context $\perp\in\mX$ such that $f^*(\perp,a)=0$ for all $a\in\mA$.
Let $\alpha,\delta\in(0,1]$ be the privacy parameter and the failure probability.

\begin{definition}[Locally private offline oracle with confidence intervals]
A sequence of offline oracles $\{\mO_n\}_{n\in\mathbb N}$ is equipped with error measures $\{\mE_{\alpha,\delta}(n)\}\subseteq\mathbb R_+$ with privacy parameter $\alpha\in(0,1]$ and failure probability $\delta\in(0,1]$,
if each $\mO_n$ could be parameterized as $\mO_n=(\mathfrak Z_1,\cdots,\mathfrak Z_n,\mathfrak O)$ such that for every $i\in[n]$, an internal entry $z_i\in\mZ$ is obtained with $z_i\sim \mathfrak Z_i(\cdot|z_1,\cdots,z_{i-1},x_i,a_i,y_i)$, and after $n$ time periods an estimate $\hat f:\mX\times\mA\to[-1,1]$ and confidence interval $\Delta:\mX\times\mA\to[-1,1]$
are obtained by $\hat f,\Delta=\mathfrak O(z_1,\cdots,z_n)$, such that the following hold:
\begin{enumerate}
\item (Privacy). For every $i\in[n]$, $z_1,\cdots,z_{i-1}\in\mZ$, $x,x'\in\mX$, $a,a'\in\mA$, $y,y'\in[-1,1]$ and measurable $Z\subseteq\mZ$, 
$\mathfrak Z_i(Z|z_1,\cdots,z_{i-1},x,a,y)\leq e^{\alpha}\mathfrak Z_i(Z|z_1,\cdots,z_{i-1},x',a',y')$;
\item (Utility). For every distribution $\mD$ supported on $\mX\times\mA$, let $(x_1,a_1),\cdots,(x_n,y_n)\overset{i.i.d.}{\sim}\mD$, and $y_1,\cdots,y_n$ realized according to Eq.~(\ref{eq:defn-reg}).
With probability $1-\delta$, the following hold:
\begin{enumerate}
\item For every $x,a\in\mX\times\mA$, $|\hat f(x,a)-f^*(x,a)|\leq\Delta(x,a)$;
\item $\mathbb E_{(x,a)\sim\mD}[\Delta(x,a)] \leq \mE_{\alpha,\delta}(n)$;
\item $\Delta(\perp,a)=0$ for all $a\in\mA$.
\end{enumerate}
\end{enumerate}
\label{defn:oracle}
\end{definition}

Intuitively, Definition \ref{defn:oracle} defines an oracle that estimates a contextual reward function $f$ from offline data in a locally private fashion.
In addition, the oracle also returns a point-wise confidence interval $\Delta(\cdot,\cdot)$ that is valid with high probability uniformly over all action-state pairs.
Using the confidence intervals, the mean-absolute deviation error of the learning procedure is also upper bounded by $\mE_{\alpha,\delta}(n)$ with high probability,
corresponding to the second property in the utility guarantees of Definition \ref{defn:oracle}.

\subsection{The action elimination algorithm}

\begin{algorithm}[t]
\caption{The action elimination (\texttt{RegCB.Elimination}) algorithm, simplified}
\label{alg:regcb}
\begin{algorithmic}[1]
\State \textbf{Input}: time horizon $T$, oracles $\{\mO_n\}_{n\in\mathbb N}$ with privacy parameter $\alpha$ and failure probability $\delta=1/T^2$, epoch schedule $n_1,n_2,\cdots$.
\State Initialize: $\hat f_1^a(\cdot)\equiv 0$ and $\Delta_1^a(\cdot)\equiv 1$ for all $a\in\mA$.
\For{$\tau=1,2,\cdots$ until $T$ time periods have elapsed}
	\State Initialize $A$ oracles $\{\mO_n^a\}_{a\in\mA}$ with $n=n_\tau$;
	\For{the next $n_\tau$ time periods}
		\State Observe $x_t\sim P_{\mX}$ and initialize $\mA_0(x_t) \gets \mA$;
		\For{$\tau'=1,2,\cdots,\tau$}
			\State $\mA_{\tau'}(x_t)\gets \{a\in\mA_{\tau'-1}: \hat f_{\tau'}^a(x_t)+\Delta_{\tau'}^a(x_t)\geq\max_{a'\in\mA_{\tau'-1}} \hat f_{\tau'}^{a'}(x_t)-\Delta_{\tau'}^{a'}(x_t)\}$;
			\label{line:defn-atau}
		\EndFor
		\State Take action $a_t\sim\mathrm{Uniform}(\mA_\tau(x_t))$ and observe realized reward $y_t$;
		\State Feed $(x_t,a_t,y_t)$ into $\mO_n^{a_t}$ and $(\perp,a,0)$ into $\mO_n^a$ for all $a\neq a_t$;
	\EndFor
	\State For every $a\in\mA$, obtain $\hat f_{\tau+1}^a(\cdot)=\hat f(\cdot,a)$ and $\Delta_{\tau+1}^a(\cdot)=\Delta(\cdot,a)$ from $\mO_n^a$.
\EndFor
\end{algorithmic}
\end{algorithm}

In this section, assuming the availability of an oracle defined in the previous section, we describe the RegCB.Elimination algorithm in \citep{foster2018practical}
and show how an upper bound on its regret could be established.
Algorithm \ref{alg:regcb} describes the pseudocode of a simplified version of the RegCB.Elimination algorithm.
The main difference and simplification comes from the availability of an explicitly constructed point-wise confidence interval $\Delta(\cdot,\cdot)$ from the input oracle.
In contrast, the work of \cite{foster2018practical} assumes no such confidence interval and therefore must resort to more complex procedures to carry out context-specific successive elimination
of actions.

The following theorem upper bounds the cumulative regret of Algorithm \ref{alg:regcb}, using MAD error bounds $\{\mE_{\alpha,\varepsilon}(\cdot)\}$ in the definition of the oracles $\{\mO_n\}$.
The theorem is stated in generality without specifying how epoch lengths $n_1,n_2,\cdots$ are set. The epoch schedules will be specified later when after we describe and analyze
the regression oracles to arrive at a final regret bound.

\begin{theorem}
The policy presented in Algorithm \ref{alg:regcb} satisfies $\alpha$-local differential privacy as defined in Definition \ref{defn:admissible}.
Furthermore, jointly with the success conditions in Definition \ref{defn:oracle} for all $\mO_n^a$ and all epochs $\tau$, the cumulative regret of Algorithm \ref{alg:regcb} is upper bounded by
$$
O\left(n_1 + A^2\sum_{\tau=2}^{\tau_0} n_\tau\mE_{\alpha,\delta}(n_{\tau-1})\right),
$$
where $\tau_0$ is the last epoch in the execution of the algorithm.
\label{thm:regcb}
\end{theorem}

Theorem \ref{thm:regcb} can be proved following roughly the same analysis in \citep{foster2018practical}, with only minor modifications. For completeness, we place a complete rigorous proof
of Theorem \ref{thm:regcb} in the supplementary material.
The high-level idea is to show that (with high probability) the action sets $\mA(x)$ always contain the optimal action $a^*(x)=\arg\max_{a\in\mA}f(x,a)$, and
that because uniform sampling over active actions is employed, an action that has not been eliminated receives at least $1/A$ of exploration probability from previous epochs.

Before proceeding, we remark that Theorem \ref{thm:regcb} only applies to bandit problems with stochastic contexts and a finite number of actions.
Neither the stochasticity nor the polynomial dependency on action space size $A$ could be removed using this algorithm, because the assumed regression oracle only operate on \emph{offline} data
and the minimum exploration probability scales inverse polynomially with $A$.
In Sec.~\ref{sec:conclusion} we further discuss this aspect and potential ways to remove stochasticity or finiteness assumptions.

\section{Layered private linear regression (LPLR)}\label{sec:lplr}

The purpose of this section is to design a locally private offline regression oracle with confidence intervals so that it provably satisfy the properties listed in Definition \ref{defn:oracle}.
As shown by Sec.~\ref{sec:negative}, this is not an easy task and cannot be accomplished by simple approaches such as input perturbation, or vanilla ordinary least squares.

Throughout this section we assume $\phi_1,\phi_2,\cdots\overset{i.i.d.}{\sim}\mD$ where $\mD$ is an unknown distribution supported on $\mX_\phi=\{\phi\in\mathbb R^d:\|\phi\|_2\leq 1\}$.
It is induced by the distribution of $\phi_i=\phi(x_i,a_i)$ where the context-action pairs $\{(x_i,a_i)\}$ are i.i.d.
Given $\phi_i\in\mX_\phi$, the realized reward $y_i$ satisfies $\mathbb E[y_i|\phi_i]=\phi_i^\top\theta^*$. We also assume for convenience that there are $2nd$ samples available 
because our method operations iteratively with $d$ iterations.
Finally, if an oracle received the empty context $x_i=\perp$, we will send to the oracle $\phi_i=(0,\cdots,0)$ and $y_i=0$.

\subsection{Key idea: Adaptive layered partitioning of $d$-dimensional space}

\begin{figure}[t]
\centering
\includegraphics[width=0.95\textwidth]{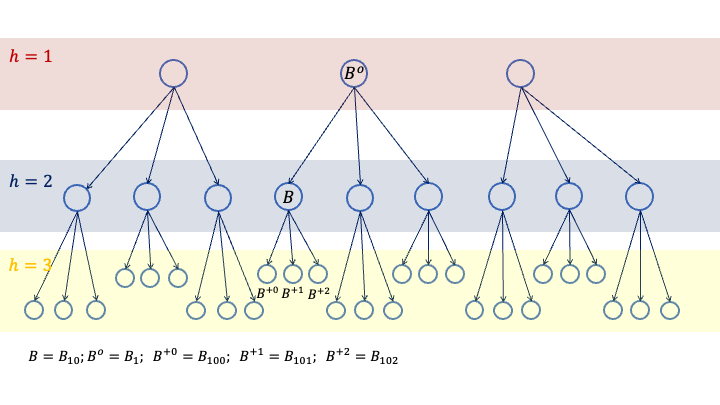}
\caption{Illustration of an example layered partitioning of $3$-dimensional space}
{\small Note: $M=3$, while the inactive bins (with maximum $k_h$ being equal to 3) are neglected. For a particular bin $B=B_{10}$ on the second layer with $k_1=1$, $k_2=0$,
bin $B^o$ is its unique parent bin on the first layer and $B^{+0}$, $B^{+1}$, $B^{+2}$ are three child bins on the immediate next layer (layer 3).}
\label{fig:bins}
\end{figure}

We first introduce the core concept of our method that involves a layered/hierarchical partition of $\mX_\phi$ into small bins.
The partition is constructed adaptively over $d$ iterations, each iteration expanding the hierarchical partition with one more layer based on the distribution $\mD$ of input covariates.
We also introduce notations for such layered partitioning that will be useful later when describing and analyzing our proposed algorithm.

Let $\beta>0$ be a small, positive parameter and denote $\gamma=T^\beta$, $M=\lceil 1/(2\beta)\rceil$.
The hierarchical partitioning tree has $d$ layers in total, {illustrated in Figure \ref{fig:bins}.}
We call a bin $B=B_{k_1\cdots k_h}$ a \emph{partitioning bin} if $k_1,\cdots,k_h<M$, because these are the bins that will be partitioned at further layers
The first layer has $M$ partitioning bins. For each partitioning bin on layer $h$, $h<d$, it is further partitioned into $M$ partitioning bins on layer $h+1$.
Therefore, layer $h$ would have a total of $M^h$ partitioning bins.
Such structure also means that we can uniquely identify any bin $B$ on layer $h$ with $B=B_{\vct k}$ where $\vct k=(k_1,\cdots,k_h)\in \{0,1,\cdots,M\}^h$.
This means that the parent bins of $B$ on layer $h-1$ is $B_{k_1\cdots k_{h-1}}$, on layer $h-2$ is $B_{k_1\cdots k_{h-2}}$, etc., and on the first layer is $B_{k_1}$.
To further simplify notations regarding parents (on previous/higher layers) and children (on later/lower layers) bins as follows. Given a bin $B=B_{\vct k}$, $\vct k=(k_1,\cdots,k_h)$
on the $h$-th layer, define the following:
\begin{enumerate}
\item $B^o=B_{k_1\cdots k_{h-1}}$ is the immediate parent bin of $B$ on the $(h-1)$th layer. $B^{-2}=(B^o)^o$ is the parent bin of $B^o$, $B^{-3}=(B^{-2})^o$ is the parent bin of $B^{-2}$, etc;
\item For $\ell_1,\ell_2,\cdots \in \{0,1,\cdots,M\}$, $B^{+\ell_1}=B_{k_1\cdots k_h\ell_1}$ is a child bin of $B$ on level $h+1$, $B^{+\ell_1\ell_2}=B_{k_1\cdots k_h\ell_1\ell_2}$ is a child bin of $B^{+\ell_1}$ on level $h+2$, etc.
\end{enumerate}

For each bin $B$ on the $h$th layer, we maintain the following statistics. These statistics or quantities are updated in a manner that will satisfy $\alpha$-local differential privacy as defined in property 1 of Definition \ref{defn:oracle}.
\begin{enumerate}
\item $\hat c_B\in\mathbb R$ and $\hat\psi_B = \hat c_B/n$, estimation of frequency and relative frequency of bin $B$;
\item $\hat\Lambda_B\in\mathbb R^{d\times d}$ and $\hat\lambda_B\in\mathbb R^d$, estimates of the conditional covariance and vector in bin $B$;
\item $\hat u_B\in\mathbb R^{d}$ a vector of unit $\ell_2$ norm and $\hat s_B\in\mathbb R_+$.
Furthermore, it is guaranteed that $\hat u_B$ is orthogonal to $\hat u_{B^o}$, $\hat u_{B^{-2}}$, etc. These quantities characterize the eigenspace partitioning;
furthermore, $\hat u_b,\hat s_b$ are set to zero for non-partitioning bins or inactive bins, which will be specified later in the algorithm;
\item $\hat\theta_B\in\mathbb R^d$, an estimated model at bin $B$ obtained via (locally private) principal component regression. It is also guaranteed that $\hat\theta_B$ is on the same
direction of $\hat u_B$, or more specifically $\hat\theta_B=\hat u_B\hat u_B^\top\hat\theta_B$;
\item $\hat\eta_B,\hat\chi_B:\mX_\phi\to[0,1]$ and $\hat\varepsilon_B\in\mathbb R$: upper bounds on partial predictive errors.
\end{enumerate}

For each vector $\phi\in\mX_\phi$ associated with $y\in[-1,1]$, there exists a unique bin at each level $h$ to which $\phi$ belongs,
and for such bin a revised vector $\phi_h\in\mX_\phi$ and a revised reward $y_h\in[-1,1]$ are associated with. They are determined by the following iterative procedure:
\begin{figure}[h]
\centering
\fbox{\begin{minipage}{0.85\textwidth}
$\phi_1=\phi$, $y_1=y$;

\textbf{For} $h=1,2,3,\cdots,d$ do:

\hspace{0.3in} $k_h=\max\{k\leq M: \|\phi_h\|_2\leq \gamma^{-k}\}$ and $\phi$ belongs to $B(h)=B_{k_1\cdots k_h}$ on level $h$;

\hspace{0.3in} $\phi_{h+1}=(I-\hat u_{B(h)}\hat u_{B(h)}^\top)\phi_h$ and $y_{h+1}=y_h-\langle \phi_h, \hat\theta_{B(h)}\rangle$.
\end{minipage}
}
\caption{Iterative procedure of partitioning $(\phi,y)$}
\label{fig:phih}
\end{figure}

We write $k_h(\phi)$ for the unique integer between $0$ and $M$ determined by the above procedure at layer $h$, and $\vct k(\phi)=(k_1(\phi),\cdots,k_d(\phi))$.
Sometimes we also write $\phi\in B$ if $\phi$ belongs to bin $B$ and $\phi\notin B$ otherwise.


\subsection{Overview of the algorithm}

\begin{algorithm}[t]
\caption{The LPLR algorithm with $nd$ samples}
\label{alg:lplr}
\begin{algorithmic}[1]
\For{$h=1,2,\cdots,d$}
	\State Initialize all bins on level $h$ so that all quantities are zero;
	\For{$i=1,2,\cdots,n$}
		\State Observe $\phi_i$, $y_i$ and let $B_{k_1\cdots k_h}$ be the bin on level $h$ determined by Figure \ref{fig:phih}; 
		\State $\textsc{LPLR-Update}(B, \phi_{i},y_{i},\alpha)$ for $B=B_{k_1\cdots k_h}$;
		\State $\textsc{LPLR-Update}(B,\perp,\perp,\alpha)$ for all $B\neq B_{k_1\cdots k_h}$ on level $h$;
	\EndFor
	\State For every $B$ on level $h$, $\hat u_B,\hat s_B,\hat\theta_B\gets\textsc{LPLR-PCR}(B,n)$;
	\State With $n$ additional samples $\{(\phi_i,y_i)\}_{i=1}^n$, $\{\hat\eta_B,\hat\chi_B\}\gets\textsc{LPLR-CI}(h,\{(\phi_i,y_i)\}_{i=1}^n,\alpha)$;
\EndFor
\State $\hat f(\cdot),\hat\Delta(\cdot)\gets\textsc{LPLR-Aggregate}(\{B\}, n)$.
\end{algorithmic}
\end{algorithm}

Algorithm \ref{alg:lplr} gives a pseudocode description of the algorithm that we use to construct the oracle defined in Sec.~\ref{subsec:oracle}.
It involves several sub-routines which will be explained in much more details later.
At a higher level, Algorithm \ref{alg:lplr} uses $2n$ i.i.d.~samples and partitions them into $d$ groups.
For each dimension $i\in[d]$, the algorithm first uses a sample of size $n$ to update cumulative statistics of all bins on layer $i$ in a locally private fashion,
via calls to the $\textsc{LPLR-Update}$ sub-routine. Afterwards, for each bin on layer $i$ that has received a sufficient number of samples,
a (noisy) principal component regression estimator is fit using the anonymized cumulative statistics, via calls to the $\textsc{LPLR-PCR}$ sub-routine.
As a last step, a fresh set of $n$ samples is used to estimate and construct point-wise confidence intervals via the \textsc{LPLR-CI} sub-routine.
After all $d$ iterations and the tree is expanded to a full $d$ layers, the \textsc{LPLR-Aggregate} sub-routine aggregates all principal component regressors
and confidence intervals on all layers to arrive at a final model estimate and confidence interval.

\subsection{The \textsc{LPLR-Update} subroutine.} \label{subsec:lplr-update}

Algorithm \ref{alg:lplr-update} gives a pseudocode description of the \textsc{LPLR-Update} sub-routine.
It updates the important cumulative statistics $\hat c_B$, $\hat\Lambda_B$ and $\hat\lambda_B$ for each bin $B$ on a particular layer in the partition tree.
The update procedure calibrates Laplace or Wishart noises so that it is locally private.
Importantly, the magnitudes of calibrated noises are heterogeneous across the different partitioned bins, which is essential to our later analysis.

The distributions of the calibrated noise variables are rigorous defined as follows. Given $m\in\mathbb N$ and $\mu\in\mathbb R$, $b>0$, a random vector $u\in\mathbb R^m$ is drawn from the vector Laplace distribution $v\sim\lap_m(\mu,b)$
if $v=(v_1,\cdots,v_m)$ and each $v_i$ are independently and identically distributed with probability density function
$$
f(v_i) = \frac{1}{2b}\exp\left\{-\frac{|v_i-\mu|}{b}\right\}, \;\;\;\;\;\;\forall v_i\in\mathbb R.
$$
Given $m\in\mathbb N$, $m>d$ and a positive definite matrix $V\in\mathbb R^{d\times d}$, a random matrix $Z\in\mathbb R^{d\times d}$ is drawn from the $d$-dimensional
Wishart distribution $W_d(m,V)$ if $Z$ is equipped with the following probability density function
$$
f(Z) = \frac{\det(Z)^{(m-d-1)/2}}{2^{md/2}\det(V)^{m/2}\Gamma_d(m/2)}\exp\left\{-\frac{1}{2}\tr(V^{-1}Z)\right\},\;\;\;\;\;\forall Z\text{ positive definite},
$$
where $\Gamma_d(a)=\pi^{d(d-1)/4}\prod_{j=1}^d\Gamma(a+(1-j)/2)$ is the multivariate Gamma function, with $\Gamma(a)=\int_0^{\infty}u^{a-1}e^{-u}\ud u$ being the standard Gamma function.

\begin{algorithm}[t]
\caption{The \textsc{LPLR-Update} subroutine}
\label{alg:lplr-update}
\begin{algorithmic}[1]
\Function{LPLR-Update}{$B,\phi,y,\alpha$}
	\State For $B=B_{k_1\cdots k_h}$, compute $\phi_h,y_h$ with input $\phi,y$ using Figure \ref{fig:phih};
	\State $\hat c_B\gets\hat c_B + \vct 1\{\phi\neq\perp\} + 3\lap_1(0, \alpha^{-1})$;
	\State $\hat\lambda_B\gets\hat\lambda_B + \vct 1\{\phi\neq\perp\}y_h\phi_h + 3\sqrt{d}\gamma^{-k_h}\lap_d(0,\alpha^{-1})$;
	\State $\hat\Lambda_B\gets\hat\Lambda_B + \vct 1\{\phi\neq\perp\}\phi_h\phi_h^\top + 3\gamma^{-2k_h}(W_d(d+1,1.5\alpha^{-1} I)-1.5(d+1)\alpha^{-1}I)$;
\EndFunction
\end{algorithmic}
\end{algorithm}

The following proposition establishes that \textsc{LPLR-Update} conforms with local privacy constraints, which is simple to prove utilizing existing results on 
the Laplace mechanism \citep{dwork2014algorithmic} and the Wishart mechanism \citep{jiang2016wishart}. Due to space constraints, its complete proof is placed in the supplementary material.
\begin{proposition}[Local privacy of \textsc{LPLR-Update}]
With $\{\hat c_B,\hat\lambda_B,\hat\Lambda_B\}_B$ considered internal stored statistics, the \textsc{LPLR-Update} procedure satisfies $\alpha$-local differential privacy
as defined in the first property of Definition \ref{defn:oracle}.
\label{prop:ldp}
\end{proposition}

We next establish the \emph{utility} of \textsc{LPLR-Update}, showing that after $n$ time periods the maintained anonymized quantities $\hat c_B,\hat\lambda_B$ and $\hat\Lambda_B$
concentrate towards certain underlying conditional expectations in each bin. Recall that for each $\phi\in\mX_\phi$, there is a unique bin $B$ to which $\phi$ belongs and $\phi_h$
is the vector recorded in that particular bin, calculated using Figure \ref{fig:phih}. For each bin $B$ on level $h$, define the following quantities:
\begin{align}
\psi_B &:= \Pr_{\mD}[\phi\in B]\in\mathbb R, \;\;\;\; \lambda_B := \mathbb E_{\mD}[y_h\phi_h|\phi\in B]\in\mathbb R^d,\;\;\;\;
\Lambda_B := \mathbb E_{\mD}[\phi_h\phi_h^\top|\phi\in B]\in\mathbb R^{d\times d}.
\label{eq:true-quantities}
\end{align}
Some remarks are in order. In the definitions of all quantities, $\phi$ is sampled from the underlying distribution $\mD$ and all quantities are defined as the \emph{conditional} expectation
of certain scales, vectors or matrices conditioned on the random event that $\phi\in B$ for the particular level-$h$ bin $B$ in consideration.
Note that the event of $\phi\in B$ depends on the partitioning and calculated quantities from bins on previous levels.
Therefore, the expectations in Eq.~(\ref{eq:true-quantities}) should be understood as conditioned on the execution of Algorithm \ref{alg:lplr} during iterations $1,2,\cdots,h-1$ as well.
The high-probability events in the following Lemma \ref{lem:lplr-update} should be understood as conditioned on the previous $h-1$ iterations too.

\begin{lemma}[Utility of \textsc{LPLR-Update}]
For any $\delta\in(0,1]$, with probability $1-\delta$ the following hold uniformly over all bins $B=B_{k_1\cdots k_h}$ on level $h$ ($\hat\psi_B := \hat c_B/\sqrt{n}$):
\begin{align*}
\big|\hat\psi_B-\psi_B\big| &\;\;\leq\;\; \frac{6.2d \ln(48d/\beta\delta)}{\alpha\sqrt{n}};\\
\left\|\frac{\hat\lambda_B}{\hat\psi_B n} - \lambda_B\right\|_2 &\;\;\leq\;\; \frac{(19d^2+29d) \gamma^{-k_h}  \ln(48d/\beta\delta)}{\alpha\psi_B \sqrt{n}} ,\\
&\;\;\;\;\;\;\text{if }\psi_B\geq  \frac{18d^2 \ln(48d/\beta\delta)}{\alpha\sqrt{n}}\text{ and }n\geq 2d\ln(24d/\beta\delta);\\
\left\|\frac{\hat\Lambda_B}{\hat\psi_B n} - \Lambda_B\right\|_\op &\;\;\leq\;\; \frac{(7(d+1)^3+29d) \gamma^{-2k_h}  \ln(48d/\delta)}{\alpha\psi_B \sqrt{n}} ,\\
&\;\;\;\;\;\; \text{if }\psi_B\geq  \frac{18(d+1)^3\ln(24d/\beta\delta)}{\alpha\sqrt{n}}  \text{ and }n\geq 2d\ln(24d/\beta\delta).\\
\end{align*}
\label{lem:lplr-update}
\end{lemma}

Lemma \ref{lem:lplr-update} is proved using standard concentration inequalities for Laplace random variables (a special case of sub-exponential random variables)
and Wishart-type quantities \citep{jiang2016wishart}. The complete proof is rather technical and is deferred to the supplementary material.

While the complete upper bounds in Lemma \ref{lem:lplr-update} are quite complex, we make two important observations.
First, both upper bounds for the concentration of $\lambda_B$ and $\Lambda_B$ is a function of $\gamma^{-k_h}$, the upper bound on the $\ell_2$ norm of all vectors
belonging to bin $B$.
This means that for bins containing weaker signals, the estimations are also more accurate because the calibrated noise also has smaller variance.
Second, the estimation errors of $\lambda_B$ and $\Lambda_B$ degrade when the sampling probability $\psi_B$ decreases, and become entirely unreliable if $\psi_B\lesssim 1/\sqrt{n}$.
This is the unique property of local privacy constraints because one copy of fresh noise variable must be incorporated regardless of whether a sample appeared in a particular bin.
It also matches the intuition behind negative result Theorem \ref{thm:ols-mse}, that no event occurring only with probability $O(1/\sqrt{n})$ could be detected with local privacy constraints.

\subsection{The \textsc{LPLR-PCR} subroutine.}

\begin{algorithm}[t]
\caption{The \textsc{LPLR-PCR} subroutine constructing principal component regression estimates}
\label{alg:lplr-pcr}
\begin{algorithmic}[1]
\Function{LPLR-PCR}{$B=B_{k_1\cdots k_h}$,$n$} \Comment{{\color{blue}with also algorithmic parameter $\kappa_1>0$}}
	\If{$B$ is inactive or {$\hat\psi_B\leq \kappa_1\gamma^2(d+1)^3/\sqrt{n}$} or $k_h=M$}
		\State Mat $B$ as inactive and \textbf{return} $\hat u_B,\hat s_B,\hat\theta_B=0$;
	\Else
		\State $\hat U_{B^o}\gets[\hat u_{B_{k_1}},\cdots,\hat u_{B_{k_1\cdots k_{h-1}}}]\in\mathbb R^{d\times(h-1)}$; \Comment{{\color{blue}concatenated orthonormal basis vectors}}
		\State $\tilde\Lambda_B\gets \arg\min_{\Lambda\in\mathbb S_d^+,\Lambda\hat U_{B^o}=0}\|\Lambda-\hat\Lambda_B/(\hat\psi_Bn)\|_\op$; \Comment{{\color{blue}$\mathbb S_d^+$ contains all $d\times d$ PSD matrices}}\label{line:proj-pcr}
		\State Let $\hat u_B\in\mathbb R^d$, $\|\hat u_B\|_2=1$, $\hat s_B\in\mathbb R^+$ be the largest eigenvector and eigenvalue of $\tilde\Lambda_B$; 
		\State $\hat\theta_B\gets\hat u_B\hat s_B^{-1}\hat u_B^\top\hat\lambda_B/(\hat\psi_Bn)$; \Comment{{\color{blue}principal component regression}}
		\State \textbf{return} $\hat u_B,\hat s_B,\hat\theta_B$;
	\EndIf
\EndFunction
\end{algorithmic}
\end{algorithm}

Algorithm \ref{alg:lplr-pcr} gives a pseudo-code description of the \textsc{LPLR-PCR} sub-routine.
The sub-routine fits a principal component regression estimator with one principal component for every bin $B$ on layer $h$ that has received a sufficient number of samples
(that is, $\hat\psi_B$ is not too small).
Note that, because $\hat\Lambda_B$ contains of Wishart noise matrices, it may not be strictly orthogonal to $\hat U_{B^o}$ and may not even be positive semi-definite.
Therefore, an additional step is carried out in Algorithm \ref{alg:lplr-pcr} to project the estimated sample covariance to the positive semi-definite cone and make sure that its span
is orthogonal to all principle directions identified in previous rounds of iterations.
Because the \textsc{LPLR-PCR} subroutine only uses $\{\hat c_B,\hat\Lambda_B,\hat\lambda_B\}$ which are already anonymized, we do not have to worry about its privacy implications thanks to the closedness-to-post-processing
property of differential privacy. 

To understand the utility of \textsc{LPLR-PCR} it is important to clarify what the returned $\{\hat\theta_B\}$ are estimating.
For bin $B=B_{k_1\cdots k_h}$ on level $h$, define the following model quantities, all of which being $d$-dimensional vectors: 
\begin{align}
\theta_B^* := \hat u_B\hat u_B^\top\theta^*, \;\;\;\; \theta_B^{\tot} := \sum_{\ell=0}^{h-1}\theta_{B^{-\ell}}^*, \;\;\;\;
\hat\theta_B^{\tot} := \sum_{\ell=0}^{h-1}\hat\theta_{B^{-\ell}}.
\label{eq:defn-thetaB}
\end{align}
More specifically, $\theta_B^*$ is the projection of the true linear model $\theta^*$ on the principle direction of $\hat u_B$,
and $\theta_B^\tot,\hat\theta_B^\tot$ are the aggregated linear models and their estimates from all past layers/iterations.
The following lemma analyzes the estimation error $\hat\theta_B-\theta_B^*$.
\begin{lemma}
Suppose Algorithm \ref{alg:lplr-pcr} is executed with parameter $\kappa_1\geq 43\alpha^{-1}\ln(48d/\beta\delta)$
and that $n\geq 2d\ln(24d/\beta\delta)$, where $\delta\in(0,1]$ is the failure probability in Lemma \ref{lem:lplr-update}.
For every bin $B=B_{k_1\cdots k_h}$ on level $h$, the estimation error can be decomposed as
$$
\hat\theta_B-\theta_B^* = -\hat u_B\hat s_B^{-1}\hat u_B^\top\mathbb E_{\mD}[\eta_{B^o}(\phi)\phi_h|\phi\in B] + \omega_B,
$$
where $\eta_{B^o}(\phi) = \phi^\top(\hat\theta_{B^o}^\tot-\theta_{B^o}^\tot)$ for $\phi\in B^o$, and $\phi_h\in\mathbb R^d$ is determined using Figure \ref{fig:phih} for $\phi\in B$.
Furthermore, jointly with the success event that all inequalities in Lemma \ref{lem:lplr-update} hold for all bins on level $h$,
we have for all bin $B=B_{k_1\cdots k_h}$ on level $h$ such that $\hat\psi_B\geq\kappa_1\gamma^2 (d+1)^3/\sqrt{n}$:
$$
\hat s_B^{-1} \leq 2d\gamma^{2k_h+2} \;\;\;\;\;\;\text{and}\;\;\;\;\;\;
\|\omega_B\|_2 \leq \frac{120(d+1)^4 \gamma^{k_h+2}  \ln(48d/\beta\delta)}{\alpha\psi_B \sqrt{n}}.
$$
\label{lem:lplr-pcr}
\end{lemma}

Lemma \ref{lem:lplr-pcr} shows that the estimation error $\hat\theta_B-\theta_B^*$ could be decomposed into two parts.
The first part is a ``bias'' term from the estimation errors from previous iterations, which is implicitly defined through the $\eta_{B^o}(\cdot)$ error function that concerns only the estimation errors
from bins $B^o$ and its parents.
The second term $\omega_B$ arises from the estimation errors of $\hat\Lambda_B,\hat\lambda_B$ and the expected quantities they estimate,
which can be upper bounded with high probability by using Lemma \ref{lem:lplr-update}.
Because we are only estimating the linear model along the principal component and all vectors belonging to bin $B$ have similar norms, it can be shown that $\hat s_B$ is not too small
and therefore inverting the eigenvalue will not degrade the estimation error significantly, unlike the negative example in Theorem \ref{thm:ols-mad}
where inverting noisy ill-conditioned sample covariance matrices leads to much larger error.

Due to space constraints, the complete proof of Lemma \ref{lem:lplr-pcr} is given in the supplementary material.
Its proof follows roughly the discussion given above, by first separating bias from previous estimates $\hat\theta_{B^o}^\tot$
and then using Lemma \ref{lem:lplr-update} together with the observation that $\hat s_B$ is not too small to establish an upper bound on $\|\omega_B\|_2$ with high probability.

\subsection{The \textsc{LPLR-CI} subroutine.}

\begin{algorithm}[t]
\caption{The \textsc{LPLR-CI} subroutine constructing point-wise confidence intervals}
\label{alg:lplr-ci}
\begin{algorithmic}[1]
\Function{LPLR-CI}{$h$,$\{(\phi_i,y_i)\}_{i=1}^n$,$\alpha$} \Comment{{\color{blue}with also algorithmic parameters $\kappa_1',\kappa_2,\kappa_3>0$}}
	\For{$i=1,2,\cdots,n$}
		\State Let $B=B_{k_1\cdots k_h}$ be determined by Figure \ref{fig:phih} with input $\phi_i,y_i$;
		\State $\varepsilon_i\gets \hat\eta_{B^o}(\phi)|\hat u_B^\top\phi_i|/\sqrt{\hat s_B}$; \Comment{{\color{blue}if $h=1$ then $\eta_{B^o}(\cdot)\equiv 0$}}
		\For{each active bin $B'=B_{k_1'\cdots k_h'}$ on level $h$}
			\State $\hat\varepsilon_{B'}\gets\hat\varepsilon_{B'} + \vct 1\{B=B'\}\varepsilon_i + \hat s_{B'}^{-1/2}\gamma^{-k_h'}\lap_1(0,1/\alpha)$;
		\EndFor
	\EndFor
	\For{all bins $B=B_{k_1\cdots k_h}$ on level $h$}
		\If{$B$ is inactive or $\hat\psi_B\leq\kappa_1'\gamma d^{3/2}/\sqrt{n}$}
			\State Set $\hat\eta_B(\cdot),\hat\chi_B(\cdot)\equiv \gamma^{-k_h}$ and mark all bins $B'$ children of $B$ as inactive;
		\Else
			\State $\bar\varepsilon_{B}\gets\frac{\hat\varepsilon_{B}}{\hat\psi_B n} + \frac{\kappa_2\gamma d^{3/2}}{\hat\psi_B n}$;
			\State $\forall\phi\in B$, $\hat\chi_B(\phi)\gets \frac{\kappa_3\gamma^{2}(d+1)^4}{\hat\psi_B\sqrt{n}} + \frac{\bar\varepsilon_B|\hat u_B^\top\phi|}{\sqrt{\hat s_B}}$, $\hat\eta_B(\phi)\gets\min\{1,\hat\eta_{B^o}(\phi)+\hat\chi_B(\phi)$\};
		\EndIf
	\EndFor
	\State \textbf{return} $\{\hat\eta_B(\cdot),\hat\chi_B(\cdot)\}$; \Comment{{\color{blue}functions for inactive bins can be arbitrary}}
\EndFunction
\end{algorithmic}
\end{algorithm}

Algorithm \ref{alg:lplr-ci} gives a pseudo-code description of the \textsc{LPLR-CI} sub-routine.
In this sub-routine, the main objective is to estimate the errors of $\hat\theta_B-\theta_B^*$, so that point-wise confidence intervals could be constructed.
When invoking Algorithm \ref{alg:lplr-ci}, a fresh set of data samples is used in order to avoid potential statistical correlation between estimates and their error estimates.
With fresh samples, when analyzing \textsc{LPLR-CI} the estimates $\hat\theta_B$ and (empirical) principle directions $\hat u_B$ could be considered as fixed and independent
from the samples collected in this sub-routine.

Let $\mB_h$ be the set of all bins on level $h$.
For a bin $B\in\mB_h$ and $\phi\in B$, let $\eta_B(\phi) =\phi^\top(\hat\theta_B^{\tot}-\theta_B^\tot)$ and $\chi_B(\phi)=\phi^\top(\hat\theta_B-\theta_B^*)$ be the cumulative
and partial predictive errors. 
The main idea that Algorithm \ref{alg:lplr-ci} construct upper bounds of $\eta_B(\cdot)$ and $\chi_B(\cdot)$ is to rely on Lemma \ref{lem:lplr-pcr},
where the first term involves historical errors that could be estimated using new samples and the second term could be upper bounded directly.
The following lemma establishes the validity of $\hat\eta_B(\cdot),\hat\chi_B(\cdot)$ error upper bounds constructed in Algorithm \ref{alg:lplr-ci} with high probability,
and also reveals an important recursion formula between average cumulative errors on neighboring layers.
\begin{lemma}
Algorithm \ref{alg:lplr-ci} satisfies $\alpha$-local differential privacy.
Fix level $h$ and assume that $|\eta_B(\phi)|\leq\hat\eta_B(\phi)$ for all $\phi\in B$ for bin $B$ on levels $1,2,\cdots,h-1$. 
Assume also that all inequalities in Lemmas \ref{lem:lplr-pcr} and \ref{lem:lplr-update} hold, which occurs with probability $1-\delta$.
Algorithm \ref{alg:lplr-ci} executed with parameters $\kappa_1'\geq 15\alpha^{-1}\ln(48d/\beta\delta)$, $\kappa_2\geq 118\alpha^{-1}\ln(24d/\beta\delta)$, $\kappa_3\geq 300\alpha^{-1}\ln(48d/\beta\delta)$ 
and $n\geq 2d\ln(24d/\beta\delta)$ satisfies the following with probability $1-\delta$ uniformly over all bins $B\in\mB_h$:
\begin{align*}
&\big|\eta_B(\phi)\big|\leq\hat\eta_B(\phi), \;\;\;\;\big|\chi_B(\phi)\big|\leq\hat\chi_B(\phi),\;\;\;\;\forall\phi\in B;\\
&\big|\phi^\top(\hat\theta_B^\tot-\theta^*)\big|\leq\hat\eta_B(\phi),\;\;\;\;\;\;\forall\phi\in B\text{ such that $B$ is inactive}.
\end{align*}
Furthermore, with probability $1-\delta$ it holds that
\begin{align*}
\sum_{B\in\mB_h}\psi_B\sqrt{\mathbb E_\mD[\hat\eta_B(\phi)^2|\phi\in B]}& \leq 2.54 \left[\sum_{B'\in\mB_{h-1}}\psi_{B'}\sqrt{\mathbb E_\mD[\hat\eta_{B'}(\phi)^2|\phi\in B']}\right] \\
&\;\;\;\;+ \frac{12M^h(\kappa_3\gamma^2(d+1)^4+\kappa_2\gamma d^2+\max\{\kappa_1\gamma^2(d+1)^3,\kappa_1'\gamma d^{3/2}\})}{\sqrt{n}},
\end{align*}
where the first term of right-hand side of the inequality is $0$ if $h=1$, and $M=\lceil 1/(2\beta)\rceil$.
\label{lem:lplr-ci}
\end{lemma}

\begin{remark}
With $\kappa_1,\kappa_1',\kappa_2,\kappa_3$ appropriately chosen and $M\leq2/\beta$, the second term on the right-hand side of the inequality in Lemma \ref{lem:lplr-ci} can be upper bounded by 
$$
C\times \frac{\gamma^2d^4\ln(48d/\beta\delta)}{\alpha\beta^h\sqrt{n}},
$$
where $C<10^4$ is an absolute numerical constant independent of any problem dependent parameters or constants.
\end{remark}

The conclusion of Lemma \ref{lem:lplr-ci} has two parts. The first part establishes that, with high probability, the constructed $\hat\chi_B(\cdot)$ and $\hat\eta_B(\cdot)$
functions are point-wise upper bounds of the errors $\chi_B(\cdot)$ and $\eta_B(\cdot)$.
The second part of Lemma \ref{lem:lplr-ci} then establishes a recursive formula regarding the ``average'' of estimation errors across all bins (weighted properly with respect to the input desnity
$\{\psi_B\}$) between two neighboring layers. This recursive formula could then be used in an iterative fashion to establish an upper bound on the mean-absolute deviation error
of the model estimates we constructed.

Due to space constraints, proof of Lemma \ref{lem:lplr-ci} is given in the supplementary material.
Its proof largely follows the error decomposition in Lemma \ref{lem:lplr-pcr} and shows that the several bias terms could be consistently estimated with data.

\subsection{The \textsc{LPLR-Aggregate} subroutine.}

\begin{algorithm}[t]
\caption{The $\textsc{LPLR-Aggregate}$ subroutine}
\label{alg:lplr-aggregate}
\begin{algorithmic}[1]
\Function{LPLR-Aggregate}{$\{B\}$,$n$}
	\For{every $\phi\in\mX_\phi$}
		\State $\hat y \gets 0$, $\phi_1\gets\phi$;
		\For{$h=1,2,\cdots,d$ or until $\hat s_{B_{h}}=0$}
			\State $k_h\gets\max\{k\leq M: \|\phi_h\|_2\leq\gamma^{-k}\}$, $B_h\gets B_{k_1\cdots k_h}$;
			\State $\hat y\gets\hat y + \langle\phi_h, \hat\theta_{B_h}\rangle$, $\phi_{h+1}\gets (I-\hat u_{B_h}\hat u_{B_h}^\top)\phi_h$;
		\EndFor
		\State $\hat f(\phi)\gets \hat y$, $\hat\Delta(\phi)\gets\hat\eta_{B_h}(\phi)$;
	\EndFor
	\State \textbf{return} $\{\hat f(\cdot),\hat\Delta(\cdot)\}$;
\EndFunction
\end{algorithmic}
\end{algorithm}

Algorithm \ref{alg:lplr-aggregate} gives a pseudo-code description of the \textsc{LPLR-Aggregate} sub-routine
which aggregates all model estimates and constructed confidence intervals on the partition tree to arrive at a final estimate.
The following lemma describes properties of the functions returned by Algorithm \ref{alg:lplr-aggregate}.

\begin{lemma}
Fix $\delta\in(0,1]$.
Let $\{\hat f(\cdot),\hat\Delta(\cdot)\}$ be the function estimates and confidence intervals returned by Algorithm \ref{alg:lplr-aggregate}.
With probability $1-2d\delta$ the following hold:
\begin{enumerate}
\item $|\hat f(\phi)-\phi^\top\theta^*|\leq\hat\Delta(\phi)$ for all $\phi\in\mX_\phi$;
\item The MAD error $\mathbb E_{\phi\sim\mD}[\Delta(\phi)]$ satisfies
$$
\mathbb E_{\phi\sim\mD}[\Delta(\phi)] \leq (2.54M)^d\times \frac{12(\kappa_3\gamma^2(d+1)^4+\kappa_2\gamma d^2+\max\{\kappa_1\gamma^2(d+1)^3,\kappa_1'\gamma d^{3/2}\})}{\sqrt{n}}.
$$
\end{enumerate}
Furthermore, with $\kappa_1,\kappa_1',\kappa_2$ and $\kappa_3$ appropriately chosen, the right-hand side of the above inequality can be upper bounded by
$$
C\times \left(\frac{2.54}{\beta}\right)^d\frac{\gamma^2d^3\ln(48d/\beta\delta)}{\alpha\sqrt{n}},
$$
where $C<\infty$ is an absolute numerical constant independent of any problem parameters.
\label{thm:lplr}
\end{lemma}

The consequences of Lemma \ref{thm:lplr} shows that the returned $\{\hat f(\cdot),\hat\Delta(\cdot)\}$ satisfy all properties listed in Definition \ref{defn:oracle},
and therefore the algorithm designed in this section could be used as a legitimate oracle that would imply regret upper bounds when incorporated in Algorithm \ref{alg:regcb}.
A more careful observation of Lemma \ref{thm:lplr} shows that the mean-absolute value error $\mathbb E_{\phi\sim\mD}[\Delta(\phi)]$ is on the order of $\tilde O(1/\sqrt{n})$,
which should lead to a regret upper bound of $\tilde O(\sqrt{T})$. This is rigorously proved in Theorem \ref{thm:final} in the next section.

Proof of Lemma \ref{thm:lplr} is given in the supplementary material. The proof builds on two important facts.
First, for any $\phi\in B$, if it belongs to active bins on all the $d$ layers, then at the last layer we must have already covered the entire linear model (that is, $\theta_B^\tot=\theta^*$
for $B$ on layer $d$ such that $\phi\in B$) because $\hat U_B^\tot$ consists of $d$ orthonormal vectors and therefore its projection operator must be the identity.
Second, iteratively applying the recursive formula in Lemma \ref{lem:lplr-ci} we can upper bound the estimation error on the last layer sequentially,
which together with all inactive layers whose cumulative errors are at most $\tilde O(1/\sqrt{n})$ completes the proof of Lemma \ref{thm:lplr}.

\subsection{Consequences for locally private linear contextual bandit}

The following theorem shows that using Algorithm \ref{alg:lplr} as the oracle satisfying Definition \ref{defn:oracle} and selecting epoch schedules $n_1,n_2,\cdots$
in Algorithm \ref{alg:regcb} carefully with geometrically increasing scalings:

\begin{theorem}
Fix arbitrary $\beta>0$ and failure probability $\delta=1/T^2$.
Instantiate the active elimination framework in Algorithm \ref{alg:regcb} with epoch lengths $n_\tau = 2^\tau n_0$ such that $n_0\asymp d^2\ln(dT/\beta)$ and the LPLR algorithm described in Algorithm \ref{alg:lplr} with appropriate algorithmic parameters $\kappa_1,\kappa_1',\kappa_2,\kappa_3$. Then the cumulative regret is upper bounded by 
$$
C\times \left(\frac{2.54}{\beta}\right)^d\frac{A^2d^5\ln^2(48dT/\beta)}{\alpha}\times T^{\frac{1}{2}+2\beta},
$$
where $C<\infty$ is a absolute numerical constant independent of all problem parameters.
\label{thm:final}
\end{theorem}

We omit the proof of Theorem \ref{thm:final} because it can be proved directly by incorporating Lemma \ref{thm:lplr} directly into Theorem \ref{thm:regcb} with the appropriate choices of parameter values and epoch lengths.
To further clarify on the regret upper bound, we make the following remark:

\begin{remark}
Tracking only dependency on $T$, the upper bound in Theorem \ref{thm:final} could be simplified to
$$
\mathrm{poly}(d,\alpha^{-1},A,\ln(dT))\times (2/\beta)^{-d} T^{1/2+2\beta}
$$
for any $\beta>0$. Specifically, with the choice of $\beta=1/\ln T$, the regret upper bound could be simplified to
$$
\mathrm{poly}(d,\alpha^{-1},A)\times \sqrt{T}\ln^d T.
$$
\end{remark}

The above remark clearly demonstrates that the regret upper bound could be arbitrarily close to the optimal scaling of $\tilde O(\sqrt{T})$, at the cost of a term that could grow
exponentially with dimension $d$ but is otherwise a poly-logarithmic term when $d$ is an arbitrarily fixed constant and $T$ grows to infinity.
It is an interesting question of whether the regret upper bound could be further improved so that the dependency on dimension $d$ is no longer exponential,
which is much less clear and may very well be impossible. We discuss this aspect in the conclusion section of this paper.

\section{Conclusion and future directions}\label{sec:conclusion}


To conclude this paper, we mention several future directions of research regarding the optimal regret of locally private contextual bandits.

\subsection{Dimension dependency} 
Because our regret upper bounds grow exponentially with the model dimension $d$, our first question regards whether such dependency could be 
improved to that of a polynomial function, or is a fundamental limit for locally private contextual bandit.
\begin{question}[polynomial dimensional dependency]
Is it possible to design a locally private linear contextual bandit algorithm with local privacy parameter $\alpha>0$,  such that for any $\delta>0$, the cumulative regret of the algorithm
is upper bounded with probability $1-\delta$ by
$$
\poly(d, 1/\alpha, \ln(T/\delta))\times \sqrt{T}?
$$
\label{que:polyd-1}
\end{question}
As a weaker version of Question \ref{que:polyd-1}, it might be easier to study whether $\sqrt{T}$ could be arbitrarily approximated with polynomial dimension dependency:
\begin{question}[weaker polynomial dimension dependency]
Is it possible to design a locally private linear contextual bandit with local privacy parameter $\alpha>0$, such that for any $\delta,\beta>0$, the cumulative regret of the algorithm is upper bounded 
with probability $1-\delta$ by
$$
\poly(d,1/\alpha,\ln(T/\delta))\times T^{1/2+\beta}?
$$
\label{que:polyd-2}
\end{question}
Note that Question \ref{que:polyd-2} allows for terms such as $\beta^{-1/\beta}$, $d^{-1/\beta}$, etc.~which is otherwise not allowed by Question \ref{que:polyd-1} since $\beta\asymp 1/\ln T$
and these terms cease to be polynomial in both $d$ and $\ln T$. Note that neither questions allow for terms like $\ln^d T$ or $\beta^{-d}$, which currently exist in our regret upper bounds.

Our current algorithmic framework is unable to solve either questions, because the number of constructed bins over $d$ layers already scale exponentially with $d$,
and the special case that each bin gets $\sim 1/\sqrt{T}$ samples would destroy the hope of any polynomial dependency on $d$.
If it is possible to extend the idea of successive principal component regression to solve either question, one must analyze bins with \emph{mixed} residue covariances from multiple bins on the previous layer, which would be very complicated because the bins no longer contain rank reduced eigenspaces and therefore it may no longer be sufficient to construct only $d$ layers in the partition tree.

On the other hand, it is possible that polynomial dimension dependency and near-optimal asymptotic scaling in $T$ cannot be obtained simultaneously.
In the work of \citep{wang2023generalized} there are also some results related to statistical learning of models with local privacy constraints that do not have polynomial dimension upper bounds,
suggesting that this is a possibility.
If a lower bound exists, it is also an interesting question to ask for which parameter $\beta_0>0$ there exists a bandit algorithm satisfying regret upper bounds in Question \ref{que:polyd-2}
for all $\beta\geq \beta_0$. Since the results of \cite{zheng2020locally} are indeed polynomial in dimension, we must have that $\beta_0\leq 1/4$.

\subsection{Large action spaces and adversarial contexts}

Our algorithm and regret analysis apply to bandit settings with stochastic contexts and finite action spaces where the regret upper bound scales polynomially with the action size.
In contrast, without privacy constraints, the LinUCB or SupLinUCB algorithm applies to large action spaces (infinite for LinUCB and $\ln|\mA|$ dependency for SupLinUCB) and adversarial contexts
(adaptively adversarial for LinUCB and obliviously adversarial for SupLinUCB).
It is interesting to study whether either or both could be relaxed:
\begin{question}
Is it possible to design a locally private linear contextual bandit algorithm with local privacy parameter $\alpha>0$, such that for any $\delta>0$, the cumulative regret of the algorithm is upper
bounded with probability $1-\delta$ by
$$
\poly(1/\alpha,\ln(|\mA|T/\delta))\times \sqrt{T}?
$$
Furthermore, is it possible to achieve the above where $x_1,\cdots,x_T\in\mX$ are selected in an adaptively or obliviously adversarial manner?
\label{que:las}
\end{question}

To address Question \ref{que:las}, it might be necessary to deal with online and adaptively generated contexts like LinUCB or SupLinUCB does.
This is a challenging task because in our current multi-layer regression framework the partitioning of hierarchical bins depends on the input distribution
and therefore cannot automatically deal with eigenspace shifts as results of distribution shifts.
In general, it is an open question whether $\tilde O(\sqrt{T})$ mean-absolute deviation error could be achieved with online data subject to local privacy constraints.

\subsection{Generalized linear models}

As a direct generalization of linear contextual models, consider the generalized linear model (GLM) where $f(x,a)=\eta(\langle\phi(x,a),\theta^*\rangle)$ for some known link function $\eta$
that satisfies $|\eta'(u)|\in[\kappa_\eta,K_\eta]$ uniformly for some constants $0<\kappa_\eta\leq K_\eta<\infty$.
The question is whether the algorithm and analysis in this paper could be extended to GLMs to achieve similar $\tilde O(\sqrt{T})$ regret upper bounds.

Unlike in the works of \cite{filippi2010parametric,li2017provably} where extension of LinUCB or SupLinUCB to generalized linear models is relatively easier,
the algorithmic framework in our paper relies on sequential principal component regression which might be challenging to extend to generalized linear models,
because the estimation errors of past principal component directions may not be completed eliminated.
The correct way to carry out successive elimination of principal components of context vectors would be the main technical challenging extending this paper to generalized linear bandits.

\subsection{General models with offline regression oracles}

The exciting recent work of \cite{simchi2022bypassing} shows that contextual bandit with finite action space and stochastic contexts could be solved
with access to an offline \emph{least-squares} regression oracle, which is much easier to construct for general function classes via standard learning theory.
No point-wise confidence intervals are needed, whose construction relies heavily on the function class and is only doable for linear or close-to-linear function classes.
It is an interesting question to extend \citep{simchi2022bypassing} to locally private bandit settings.

In \citep{simchi2022bypassing} the offline regression oracle is defined so that the mean-square error is small. As shown by Theorem \ref{thm:ols-mse} of this paper,
the mean-square error is not a good measure with local privacy constraints, yielding sub-optimal regret even for linear functions.
Unfortunately, the work of \citep{simchi2022bypassing} as well as the prior work of \cite{agarwal2012contextual} rely on the following Cauchy-Schwarz inequality to 
decompose errors under distribution shift:
$$
\mathbb E\left[\Delta(x)\right] \leq \sqrt{\mathbb E\left[\frac{1}{\pi(x)}\right]}\times \sqrt{\mathbb E\left[\pi(x)\Delta(x)^2\right]},
$$
which must give rise to the mean-square error measure on the right-hand side of the inequality.
It is unclear whether this fundamental step could be revised so that the mean-absolute deviation measure could be used as an upper bound.

\bibliographystyle{ormsv080}
\bibliography{refs}

\newpage
\ECSwitch

\ECHead{Supplementary materials}

\section{Proofs of lower bounds}

\subsection{Proof of Theorem \ref{thm:ols-mse}}

For $d=2$ and a particular $n\geq 1$, construct $\phi\sim\mD$ as follows:
$$
\phi = \left\{\begin{array}{ll} e_1,& \text{with probability $1-c/\sqrt{n}$},\\
e_2,& \text{with probability $c/\sqrt{n}$},\end{array}\right.
$$
where $e_1,e_2$ are the standard basis vectors of two dimensions, and $c>0$ is a parameter to be determined later. 
It then follows that
$$
\Lambda_{\mD} = \left[\begin{array}{cc} 1-\frac{c}{\sqrt{n}}& 0\\ 0& \frac{c}{\sqrt{n}}\end{array}\right].
$$

Consider two hypothesis below:
\begin{align}
H_0:&\;\;\;\; \theta_0^* = (0, 0)\\
H_1:&\;\;\;\; \theta_1^* = (0, 1).
\end{align}
Let $P_{\nu,i}$, $\nu\in\{0,1\}$, $i\in[n]$ be the distribution of $(\phi_i,y_i)$ under $\mD$ and hypothesis $\nu$. Because the distribution of $\phi$ is the same and the distribution of $y$ is only different
when $\phi_i=e_2$, it holds that
$$
\|P_{0,i}-P_{1,i}\|_{\tv} \leq \frac{c}{\sqrt{n}}.
$$
Now let $Q_\nu^\pi$ be the distribution of internal storage of any (potentially adaptive) $\alpha$-locally private policy $\pi$ over $n$ samples under hypothesis $\nu\in\{0,1\}$.
Theorem 1 of \citep{duchi2018minimax} and the additivity of the KL-divergence implies that
\begin{align}
\kl(Q_0^\pi\|Q_1^\pi)+\kl(Q_1^\pi\|Q_0^\pi) &\leq n\times \min\{4,e^{2\alpha}\}(e^\alpha-1)^2 \|P_{0,1}-P_{1,1}\|_\tv^2 \leq {4c^2(e^\alpha-1)^2}.
\end{align}
The Pinsker's inequality then yields that
\begin{align}
\|Q_0^\pi-Q_1^\pi\|_{\tv}\leq \sqrt{\frac{1}{2}\kl(Q_0^\pi\|Q_1^\pi)} \leq {\sqrt{2}c(e^\alpha-1)}.
\label{eq:proof-mse-1}
\end{align}
With the selection that $c=1/(4\sqrt{2}(e^{\alpha}-1))$, the right-hand side of Eq.~(\ref{eq:proof-mse-1}) is upper bounded by $1/4$, which means that any $\alpha$-locally private
procedure fails to distinguish between $H_0$ and $H_1$ with constant probability. This implies a lower bound of 
$$
(\theta_0^*-\theta_1^*)^\top\Lambda_{\mD}(\theta_0^*-\theta_1^*) = \frac{c}{\sqrt{n}} = \Omega\left(\frac{1}{(e^\alpha-1)\sqrt{n}}\right),
$$
which is to be proved.

\subsection{Proof of Theorem \ref{thm:ols-mad}}

Fix $d=2$ and any $n\geq 1$. Construct the distribution $\phi\sim\mD$ as follows:
$$
\phi_i = \left\{\begin{array}{ll} (1/2, c n^{-1/3})& \text{ with probability $1/2$},\\ (1/2, -cn^{-1/3})& \text{ with probability $1/2$},\end{array}\right.
$$
where $c\in(0,1/2]$ is a parameter to be determined later. Consider the following two hypothesis:
\begin{align*}
H_1:& \;\;\;\;\;\; \theta_1^* = (1/2, 0);\\
H_2:& \;\;\;\;\;\; \theta_2^* = (1/2,1/2).
\end{align*}
Let $\tilde\phi_i = \phi_i + \lap_2(0, \alpha^{-1})$ and $\tilde y_i=y_i+\lap_1(0,\alpha^{-1})$ be the anonymized statistics using input perturbation and the Laplace mechanism,
where $y_i=\phi_i^\top\theta^*$. Let $z_i = (\tilde\phi_i,\tilde y_i)\in\mathbb R^3$ be the concatenated observation vectors. Then we have the following:
\begin{align*}
\text{Under $H_1$:}& \;\;\;\; z_i \overset{i.i.d.}{\sim} \frac{1}{2}\lap_3\left(\frac{1}{2},cn^{-1/3},0; \alpha^{-1}\right) + \frac{1}{2}\lap_3\left(\frac{1}{2},-cn^{-1/3},0;\alpha^{-1}\right);\\
\text{Under $H_2$:}&\;\;\;\; z_i \overset{i.i.d.}{\sim} \frac{1}{2}\lap_3\left(\frac{1}{2},cn^{-1/3}, cn^{-1/3}; \alpha^{-1}\right) + \frac{1}{2}\lap_3\left(\frac{1}{2},-cn^{-1/3},-cn^{-1/3};\alpha^{-1}\right).
\end{align*}
More specifically, under either $H_1$ or $H_2$, the observed vectors $\{z_i\}_{i=1}^n$ follow a \emph{mixture} of Laplace distributions. To facilitate our analysis, we also introduce the following
auxiliary distribution, though it does not correspond to any hypothesis or difficult examples constructed in this proof:
$$
\text{Under $H_0$:}\;\;\;\; z_i\overset{i.i.d.}{\sim} \lap_3\left(\frac{1}{2},0,0;\alpha^{-1}\right).
$$

The following technical lemma is essential to our proof:
\begin{lemma}
Let $\mu\in\mathbb R^2$ be a 2-dimensional vector such that $\|\mu\|_1\leq 1$. Let $P_0=\lap_2(0,0;\alpha^{-1})$ be a centered bivariate Laplace distribution with independent components,
and $P_{\pm\mu}=0.5\lap_2(\mu;\alpha^{-1})+0.5\lap_2(-\mu;\alpha^{-1})$ be a mixture of bivariate Laplace distributions with the same variance and symmetric centers.
Then
$$
\kl(P_{\pm\mu}\|P_0) \leq 16.7\alpha^3\|\mu\|_1^3.
$$
\label{lem:mixture-lap}
\end{lemma}
\begin{proof}{Proof of Lemma \ref{lem:mixture-lap}.}
Let $\mu=(\mu_1,\mu_2)$. Following the definition of Laplacian random variables, the PDF of $\lap_2(\mu;\alpha^{-1})$ can be written as
$$
p(u,v) = \frac{\alpha^2}{4}\exp\left\{-\alpha|u-\mu_1|-\alpha|v-\mu_2|\right\},\;\;\;\;\;\;\forall u,v\in\mathbb R.
$$
The likelihood ratio between $P_0$ and $P_{\pm\mu}$ can be analyzed as
\begin{align}
\frac{\ud P_{\pm\mu}(u,v)}{\ud P_0(u,v)} &= \frac{\frac{1}{2}\exp\{-\alpha|u-\mu_1|-\alpha|v-\mu_2|\} + \frac{1}{2}\exp\{-\alpha|u+\mu_1|-\alpha|v+\mu_2|\}}{\exp\{-\alpha|u|-\alpha|v|\}}\nonumber\\
&= \frac{1}{2}\exp\left\{-\alpha\big(|u-\mu_1|-|u|\big)-\alpha\big(|v-\mu_2|-|v|\big)\right\} \nonumber\\
&\;\;\;\;+\frac{1}{2}\exp\left\{-\alpha\big(|u+\mu_1|-|u|\big)-\alpha\big(|v+\mu_2|-|v|\big)\right\}.\label{eq:proof-mixture-1}
\end{align}

Discuss two cases.
\paragraph{Case 1: $u\geq |\mu_1|\wedge v\geq|\mu_2|$ or $u\leq-|\mu_1|\wedge v\leq -|\mu_2|$.} In this case, the right-hand side of Eq.~(\ref{eq:proof-mixture-1}) is reduced to
$$
\frac{1}{2}\exp\{-\alpha\mu_1-\alpha\mu_2\} + \frac{1}{2}\exp\{\alpha\mu_1+\alpha\mu_2\}.
$$
Because $\|\mu\|_1=|\mu_1|+|\mu_2|\leq 1\leq\alpha^{-1}$ (since $\alpha\in(0,1]$), by Taylor expansion with Lagrangian remainders it holds that
\begin{align*}
\left|e^{-\alpha\mu_1-\alpha\mu_2}-(1-\alpha\mu_1-\alpha\mu_2)\right| &\leq \frac{e}{2}\alpha^2(\mu_1+\mu_2)^2;\\
\left|e^{\alpha\mu_1+\alpha\mu_2}-(1+\alpha\mu_1+\alpha\mu_2)\right| &\leq \frac{e}{2}\alpha^2(\mu_1+\mu_2)^2.
\end{align*}
Subsequently, in this case
\begin{align}
\left|\frac{\ud P_{\pm\mu}(u,v)}{\ud P_0(u,v)}- 1\right| \leq \frac{e}{2}\alpha^2(\mu_1+\mu_2)^2.\label{eq:proof-mixture-2}
\end{align}

\paragraph{Case 2: $u\geq |\mu_1|\wedge v\leq -|\mu_2|$ or $u\leq-|\mu_1|\wedge v\geq |\mu_2|$.} In this case, the right-hand side of Eq.~(\ref{eq:proof-mixture-1}) is reduced to 
$$
\frac{1}{2}\exp\{-\alpha\mu_1+\alpha\mu_2\} + \frac{1}{2}\exp\{\alpha\mu_1-\alpha\mu_2\}.
$$
Using the same analysis as the previous case, we can show that Eq.~(\ref{eq:proof-mixture-2}) holds too in this case.

\paragraph{Case 3: $|u|\leq|\mu_1|$ or $|v|\leq|\mu_2|$.} In this case, define the following quantities:
$$
\delta_1^+ := |u-\mu_1|-|u|, \;\;\;\;\delta_2^+ := |v-\mu_2|-|v|,\;\;\;\;\delta_1^- := |u+\mu_1|-|u|,\;\;\;\;\delta_2^- := |v+\mu_2|-|v|.
$$
It is easy to verify that $|\delta_1^+|,|\delta_1^-|\leq |\mu_1|$ and $|\delta_2^+|,\delta_2^-|\leq|\mu_2|$ hold for all $u,v\in\mathbb R$.
Taylor expansion and the fact that $|\mu_1|+|\mu_2|\leq1\leq \alpha^{-1}$ then yield
\begin{align}
\big|e^{-\alpha\delta_1^+-\alpha\delta_2^+}-1\big|\;\;&\leq\;\; e\alpha(|\mu_1|+|\mu_2|);\nonumber\\
\big|e^{-\alpha\delta_1^--\alpha\delta_2^-}-1\big|\;\;&\leq\;\; e\alpha(|\mu_1|+|\mu_2|).\nonumber
\end{align}
Subsequently, in this case
\begin{align}
\left|\frac{\ud P_{\pm\mu}(u,v)}{\ud P_0(u,v)}- 1\right| \leq e\alpha(|\mu_1|+|\mu_2|).
\label{eq:proof-mixture-3}
\end{align}
Additionally, using the union bound and the PDF of Laplace distributions, we have that
\begin{align}
P_0[\text{Case 3}] &\leq P_0[|u|\leq|\mu_1|] + P_0[|v|\leq|\mu_2|] \leq 2\times \frac{\alpha^2}{4}\times (2|\mu_1|+2|\mu_2|)\leq 2\alpha(|\mu_1|+|\mu_2|).
\label{eq:proof-mixture-4}
\end{align}

Combining Eqs.~(\ref{eq:proof-mixture-2},\ref{eq:proof-mixture-3},\ref{eq:proof-mixture-4}), the $\chi^2$-divergence between $P_{\pm\mu}$ and $P_0$ can be upper bounded as
\begin{align}
\chi^2(P_{\pm\mu},P_0) &=\int_{\mathbb R^2}\left(\frac{\ud P_{\pm\mu}}{\ud P_0}-1\right)^2\ud P_0(u,v)\nonumber\\
&\leq \frac{e^2}{4}\alpha^4(\mu_1+\mu_2)^4 + 2\alpha(|\mu_1|+|\mu_2|)\times e^2\alpha^2(|\mu_1|+|\mu_2|)^2\nonumber\\
&= \frac{e^2\alpha^4}{4}\|\mu\|_1^4 + 2e^2\alpha^3\|\mu\|_1^3\leq 16.7\alpha^3\|\mu\|_1^3,\label{eq:proof-mixture-5}
\end{align}
where the last inequality holds because $\|\mu_1\|=|\mu_1|+|\mu_2|\leq 1\leq \alpha^{-1}$.
Finally, it is well-known that 
$$
\kl(P\|Q) \leq \chi^2(P,Q)
$$
for any $P\ll Q$. This completes the proof. $\square$
\end{proof}

We are now ready to prove Theorem \ref{thm:ols-mad}. Let $P_0,P_1,P_2$ be the distributions of $z_i$ under $H_0,H_1$ and $H_2$, respectively.
Let $P_0^n,P_1^n,P_2^n$ be the product measure of the distributions of $\{z_i\}_{i=1}^n$. By Pinsker's inequality, sub-additivity of TV-distance and additivity of the KL-divergence, it holds that
\begin{align}
\|P_1^n-P_2^n\|_\tv &\leq \|P_1^n-P_0^n\|_\tv + \|P_0^n-P_2^n\|_\tv\leq \sqrt{\frac{n}{2}}\left(\sqrt{\kl(P_0\|P_1)} + \sqrt{\kl(P_0\|P_2)}\right).
\label{eq:proof-mad-2}
\end{align}
Now apply Lemma \ref{lem:mixture-lap} with $\mu=(cn^{-1/3},0)$ for $\kl(P_0\|P_1)$ and $\mu=(cn^{-1/3},cn^{-1/3})$ for $\kl(P_0\|P_2)$, we have that
\begin{align}
\|P_1^n-P_2^n\|_\tv &\leq 11.1\sqrt{n}\times \sqrt{\frac{\alpha^3 c^3}{n}} = 11.1c^{3/2}\alpha^{3/2}.
\label{eq:proof-mad-3}
\end{align}
With the parameter $c$ chosen as $0.07$, the right-hand side of Eq.~(\ref{eq:proof-mad-3}) is upper bounded by $1/4$, meaning that any procedure would fail to distinguish between $P_1$ and $P_2$
with constant probability on $n$ samples. Subsequently, the mean-absolute deviation error is lower bounded by
$$
\Omega(1)\times \mathbb E_{\phi\sim\mD}\left[|\phi^\top(\theta_1^*-\theta_2^*)\big|\right] = \Omega(cn^{-1/3}),
$$
which is to be proved.

\section{Proof of Theorem \ref{thm:regcb}}

Throughout this proof we shall operate jointly with the success conditions in Definition \ref{defn:oracle} for all $\mO_n^a$ and all epochs $\tau$. 
We shall also drop the subscripts of $\alpha,\delta$ and simply abbreviate $\mE(n)$ for $\mE_{\alpha,\delta}(n)$.

For every epoch $\tau$ and action $a\in\mA$, let $\mO_\tau^a$ be the regression oracle initialized and operated during epoch $\tau$ for $n_\tau$ time periods.
Let $\mD_\tau^a$ be the distribution of the context-action pairs received by $\mO_\tau^a$. More specifically, $\mD_\tau^a$ is the law of the random variable
$$
\vct 1\{a_t=a\}(x_t,a) + \vct 1\{a_t\neq a\}(\perp,a).
$$
Let also $\tilde\mD_\tau^a$ be the distribution of the random variable
$$
\vct 1\{a\in\mA_\tau(x_t)\}(x_t,a) + \vct 1\{a_t\notin\mA_\tau(x_t)\}(\perp,a).
$$

For every $x\in\mX$ let $a^*(x)=\arg\max_{a\in\mA} f^*(x,a)$ be the optimal action. We first establish the following lemma, showing that $a^*(x)$ will never be eliminated from the active set
and the active set only contains actions that are approximately optimal.
\begin{lemma}
For every epoch $\tau$ and every context $x\in\mX$, the following hold:
\begin{enumerate}
\item $a^*(x)\in\mA_\tau(x)$;
\item For every $a\in\mA_\tau(x)$, $f^*(x,a^*(x))-f^*(x,a)\leq 2(\Delta_\tau^a(x)+\Delta_\tau^{a^*(x)}(x)) \leq 2\sum_{a\in\mA_\tau(x)}\Delta_\tau^a(x)$.
\end{enumerate}
\label{lem:regcb-valid}
\end{lemma}
\begin{proof}{Proof of Lemma \ref{lem:regcb-valid}.}
We use mathematical induction on $\tau$ to prove this lemma. The base case is $\tau=1$, for which Lemma \ref{lem:regcb-valid} trivially holds because $\mA_1(x)\equiv \mA$
and $\Delta_1^a(x)\equiv 1$.
We now assume the lemma holds for epoch $\tau-1$, and prove that it also holds for $\tau$.

By inductive hypothesis, $a^*(x)\in\mA_{\tau-1}(x)$. Furthermore, for every $a\in\mA_{\tau-1}(x)$, using property 2-(a) in Definition \ref{defn:oracle}, it holds that
\begin{equation}
\big|\hat f_\tau^a(x)-f^*(x,a)\big| \leq \Delta_\tau^a(x).
\label{eq:proof-regcb-1}
\end{equation}
Subsequently, 
\begin{align}
\hat f_\tau^{a^*(x)}(x)+\Delta_\tau^{a^*}(x) &\geq f^*(x,a^*(x))
= \max_{a\in\mA_{\tau-1}(x)}\{f^*(x,a)\}\nonumber\\
&\geq \max_{a\in\mA_{\tau-1}(x)}\{\hat f_\tau^a(x)-\Delta_\tau^a(x)\},
\end{align}
which implies that $a^*(x)\in\mA_\tau(x)$ from the definition of $\mA_\tau(x)$ on Line \ref{line:defn-atau} from Algorithm \ref{alg:regcb}.
This proves the first property for epoch $\tau$.

We next turn to the second property. For every $a\in\mA_\tau(x)$, it holds that
\begin{align}
f^*(x,a) & \geq \hat f_\tau^a(x)-\Delta_\tau^a(x) {\geq} \max_{a'\in\mA_{\tau-1}(x)} \hat f_\tau^{a'}(x)-\Delta_\tau^{a'}(x)-2\Delta_\tau^a(x)\label{eq:proof-regcb-2}\\
&\geq \hat f_{\tau}^{a^*(x)}(x)-\Delta_\tau^{a^*(x)}(x)-2\Delta_\tau^a(x)\label{eq:proof-regcb-3}\\
&\geq f^*(x,a^*(x)) - 2\Delta_\tau^{a^*(x)}(x) - 2\Delta_\tau^a(x),\label{eq:proof-regcb-4} 
\end{align}
where the first inequality in Eq.~(\ref{eq:proof-regcb-2}) and Eq.~(\ref{eq:proof-regcb-4}) hold by property 2-(a) in Definition \ref{defn:oracle};
the second inequality in Eq.~(\ref{eq:proof-regcb-2}) holds from the definition of $\mA_\tau(x)$;
Eq.~(\ref{eq:proof-regcb-3}) holds because $a^*(x)\in\mA_{\tau-1}(x)$ thanks to the inductive hypothesis. This proves the second property in Lemma \ref{lem:regcb-valid}
in epoch $\tau$. $\square$
\end{proof}

Now consider a particular epoch $\tau$ and let $p_\tau(\cdot|x)$ be the policy implemented for all time periods in epoch $\tau$. Using the second property of Lemma \ref{lem:regcb-valid}, the regret at each time period is upper bounded by
\begin{align}
\mathbb E_{x_t\sim P_{\mX}}&\left[f^*(x_t,a^*(x_t))-f^*(x_t,a_t)|a_t\sim p_\tau(\cdot|x_t)\right]
\leq \mathbb E_{x_t\sim P_{\mX}}\left[\sum_{a\in\mA_\tau(x_t)}\Delta_\tau^a(x_t)\right]\nonumber\\
&= \sum_{a\in\mA}\mathbb E_{x_t\sim P_{\mX}}\left[\vct 1\{a\in\mA_\tau(x_t)\}\Delta_\tau^a(x_t)\right]
= \sum_{a\in\mA}\mathbb E_{\tilde\mD_\tau^a}[\Delta_\tau^a(x)]\label{eq:proof-regcb-4}\\
&\leq \sum_{a\in\mA}\mathbb E_{\tilde\mD_{\tau-1}^a}[\Delta_\tau^a(x)] = \sum_{a\in\mA}\mathbb E_{x_t\sim P_{\mX}}\left[\vct 1\{a\in\mA_{\tau-1}(x_t)\}\Delta_\tau^a(x_t)\right]\label{eq:proof-regcb-5}\\
&\leq \sum_{a\in\mA} A\times \mathbb E_{x_t\sim P_{\mX}}\left[p_{\tau-1}(a|x_t)\Delta_\tau^a(x_t)\right]\label{eq:proof-regcb-6}\\
&= A\sum_{a\in\mA}\mathbb E_{\mD_{\tau-1}^a}\left[\Delta_\tau^a(x_t)\right] \leq A\sum_{a\in\mA}\mE(n_{\tau-1}) = A^2\mE(n_{\tau-1}).\label{eq:proof-regcb-7}
\end{align}
Here, the second inequality in Eq.~(\ref{eq:proof-regcb-4}) holds by following the definition of $\tilde\mD_\tau^a$ and the fact that $\Delta_\tau^a(\perp,a)=0$ thanks to property 2-(c)
of Definition \ref{defn:oracle}; 
the first inequality in Eq.~(\ref{eq:proof-regcb-5}) holds because $\mA_\tau(x)\subseteq\mA_{\tau-1}(x)$ and therefore $a\in\mA_\tau(x)$ implies $a\in\mA_{\tau-1}(x)$ by construction of these active sets;
Eq.~(\ref{eq:proof-regcb-6}) holds because for every $x$ and active action $a\in\mA_{\tau-1}(x)$, the probability of taking action $a$ is at least $1/|\mA_{\tau-1}(x)|\leq 1/A$.
The second inequality in Eq.~(\ref{eq:proof-regcb-7}) holds by property 2-(b) of Definition \ref{defn:oracle}.
Theorem \ref{thm:regcb} is then proved by summing over all epochs $\tau$.

\section{Proofs of results regarding the LPLR method}

This section presents complete proofs to all technical lemmas and propositions omitted in Sec.~\ref{sec:lplr}.

\subsection{Proof of Proposition \ref{prop:ldp}}
Recall that on level $h$ there are a total of $L=M^h$ bins, where $M=\lceil 1/2\beta\rceil$.
For every $\phi$, let $\vct c\in\{0,1\}^L$ be such that $c_B=\vct 1\{\phi\in B\}$ for bin $B$. Because each $\phi$ belongs to one and exactly one bin on level $h$,
the $\ell_1$-sensitivity of $\vct c$ is one. Hence, $\hat c_B$ satisfies $\alpha/3$-differential privacy thanks to the Laplace mechanism \citep{dwork2014algorithmic}.

In a similar vein, define $\vct\mu\in\mathbb R^{Ld}$ such that $\vct\lambda_B=\gamma^{-k_h}y_h\phi_h$ if $\phi\in B$ and $\vct\lambda_B=0$ otherwise.
By definition, $|y_h|\leq 1$, $\|\phi_h\|_2\leq \gamma^{-k_h}$ and therefore $\|\vct\mu\|_1\leq \|\vct\lambda_B\|_1\leq \sqrt{d}\|\vct\lambda_B\|_2\leq \sqrt{d}$.
This means the $\ell_1$-sensitivity of $\vct\mu$ is $\sqrt{d}$. Hence, $\hat\lambda_B$ satisfies $\alpha/3$-differential privacy again thanks to the Laplace mechanism.

Finally, for $\hat\Lambda_B$, note that $\|\gamma^{-k_h}\phi_h\|_2\leq 1$. The $\alpha/3$-differential privacy of $\hat\Lambda_B$ is then a consequence
of the Wishart mechanism proposed and analyzed in \citep{jiang2016wishart}.

\subsection{Proof of Lemma \ref{lem:lplr-update}}

For simplicity we fix $\delta_0\in(0,1]$ a particular bin $B=B_{k_1\cdots k_h}$. The lemma \ref{lem:lplr-update} holds uniformly over all bins $B$ except for a failure probability of $\delta=M^h\delta_0\leq 2\beta^{-d}\delta_0$.

We first focus on $\hat\psi_B$. The Laplace distribution $X\sim\lap_1(\alpha^{-1})$ is a sub-exponential distribution with parameters $\nu=2\alpha^{-1}$ and $b=2\alpha^{-1}$, meaning that $\mathbb E[e^{\lambda X}]\leq e^{\nu^2\lambda^2/2}$
for all $|\lambda|\leq 1/b$. The sample mean of $n$ such Laplace random variables then follows a centered sub-exponential distribution with parameters $\nu=2\sqrt{n}\alpha^{-1}$ and $b=2\alpha^{-1}$. 
Let $\tilde\psi_B = n^{-1}(\sum_{i=1}^n\vct 1\{\phi_i\in B\})$. By Bernstein's inequality, for any $\epsilon>0$ it holds that
\begin{align}
\Pr[\big|\hat\psi_B-\tilde\psi_B\big|>\epsilon] \leq 2\exp\left\{-\frac{n^2\epsilon^2}{2(4n\alpha^{-2}+2n\alpha^{-1}\epsilon)}\right\}.
\label{eq:proof-lplr-update-1}
\end{align}
Additionally, by Hoeffding's inequality it holds for every $\epsilon>0$ that
\begin{align}
\Pr[\big|\tilde\psi_B-\psi_B\big|>\epsilon]\leq 2\exp\left\{-\frac{2n^2\epsilon^2}{n}\right\}.
\label{eq:proof-lplr-update-2}
\end{align}
where $\psi_B = \Pr_{\mD}[\phi\in B]$.
Equating the right-hand sides of both Eqs.~(\ref{eq:proof-lplr-update-1},\ref{eq:proof-lplr-update-2}) with $\delta_0/6$, we obtain with probability $1-\delta_0/3$ that
\begin{align}
\big|\hat\psi_B-\psi_B\big| \leq \frac{4\ln(12/\delta_0)}{\alpha n} + \frac{\sqrt{2\ln(12/\delta_0)}}{\alpha\sqrt{n}} + \sqrt{\frac{\ln(12/\delta_0)}{2n}} \leq \frac{6.2\ln(12/\delta_0)}{\alpha\sqrt{n}}.
\label{eq:proof-lplr-update-3}
\end{align}

    
    We next turn to the concentration inequality involving $\hat\lambda_B$. We have that
    \begin{align}
    \left\|\frac{\hat\lambda_B}{\hat\psi_B n} - \lambda_B\right\|_2 \;\;=    \left\|\frac{\hat\lambda_B}{\hat\psi_B n} -\frac{\hat\lambda_B}{\psi_B n}\right\|_2+\left\|\frac{\hat\lambda_B}{\psi_B n} - \lambda_B\right\|_2.
    \label{eq:proof-lplr-update-4}
    \end{align}
   To upper bound the first term in Eq.~(\ref{eq:proof-lplr-update-4}), we need to upper bound the norm of $\hat\lambda_B$ with high probability.
   By definition,  $\hat{\lambda}_B=\sum_{i=1}^n \vct 1 \{\phi_i \in B\} y_{ ih} \phi_{ ih} +3 \sqrt{d} \gamma^{-k_h}\sum_{i=1}^n X_i$, where $X_i \sim \lap_d(\alpha^{-1})$.
    Separately calculating these two terms, we have that
    \begin{align}
        \left\|\sum_{i=1}^n \vct 1 \{\phi_i \in B\} y_{ ih} \phi_{ ih} \right\|_2\le \gamma^{-k_h} \sum_{i=1}^n \vct 1 \{\phi_i \in B\} \le 2 \gamma^{-k_h}\psi_B n
    \end{align}
with at least probability $1-e^{-2n}$, where the first inequality comes from $\|y_h \phi_h\|_2 \le \gamma^{-k_h}$, and the second inequality is given by Chernoff-Hoeffding bound. 
With the condition that $n\geq 2 \ln({12}/{\delta_0})$, the probability is at least $1-\frac{\delta_0}{12}$.
As for the second term, denote the $j^{th}$ coordinate of a $d$-dimension vector $X$ as $(X)_j$, we know that 
$\|\sum_{i=1}^n X_i \|_2=\sqrt{\sum_{j=1}^d (\sum_{i=1}^n X_i)_j^2}$.
For each coordinate $(\sum_{i=1}^n X_i)_j$, it is sub-exponential with $\nu=2 \sqrt{n}\alpha^{-1}$ and $b=2\alpha^{-1}$. Using union bound over $d$ coordinates and Bernstein inequality for centered sub-exponential random variables, we have
\begin{align}
     \Pr\left[\max_{j=1 \cdots, d } \bigg|(\sum_{i=1}^n X_i)_j\bigg| \ge  6\alpha^{-1}\sqrt{n}\ln({24d}/{\delta_0})\right] \le \frac{\delta_0}{12}
     \label{eq:proof-lplr-update-4half}
\end{align}
Subsequently,
\begin{align}
   \left\| 3 \sqrt{d} \gamma^{-k_h}\sum_{i=1}^n X_i \right\|_2
    \le 18d \gamma^{-k_h}\alpha^{-1}{\sqrt{n}}\ln({24d}/{\delta_0}).\label{eq:proof-lplr-update-5}
\end{align}
with probability at least $1-\frac{\delta_0}{12}$.
Combining Eqs.~(\ref{eq:proof-lplr-update-4},\ref{eq:proof-lplr-update-5}), we have 
\begin{align}
\|\hat{\lambda}_B\|_2 \le 2 \gamma^{-k_h}\psi_B n +18\alpha^{-1}d \gamma^{-k_h}{\sqrt{n}}\ln({24d}/{\delta_0})
\label{eq:proof-lplr-update-6}
\end{align}
 with probability $1-\frac{\delta_0}{6}$.
With the condition that $\psi_B \ge  \frac{18d}{\alpha\sqrt{n}} \ln(\frac{24d}{\delta_0})$, Eq.~(\ref{eq:proof-lplr-update-6}) is reduced to 
\begin{align}
\|\hat{\lambda}_B\|_2 \le 3 \gamma^{-k_h}\psi_B n.
\label{eq:proof-lplr-update-7}
\end{align}. 
On the other hand, the condition $\psi_B \ge  \frac{18d}{\alpha\sqrt{n}} \ln(\frac{24d}{\delta_0})$ together with Eq.~(\ref{eq:proof-lplr-update-3})
implies that $\hat{\psi}_B \psi_B  \geq 0.65 \psi_B^2$, and subsequently 
\begin{align}
    \left|\frac{1}{\hat{\psi}_B}-\frac{1}{\psi_B}\right| \leq \frac{|\hat\psi_B-\psi_B|}{0.65\psi_B^2}\leq  \frac{{9.54\ln(12/\delta_0)}}{\alpha\psi_B^2\sqrt{n}}.
    \label{eq:proof-lplr-update-8}
\end{align}
Combining Eqs.~(\ref{eq:proof-lplr-update-7},\ref{eq:proof-lplr-update-8}), we have that
     \begin{align}       \left\|\frac{\hat\lambda_B}{\hat\psi_B n} -\frac{\hat\lambda_B}{\psi_B n}\right\|_2 &\leq \left\|\frac{\hat{\lambda}_B}{n}\right\|_2 \left|\frac{1}{\hat{\psi_B}}-\frac{1}{\psi_B}\right|\le \frac{3}{n}\left(\gamma^{-k_h}\psi_B n\right) \frac{{9.54\ln(12/\delta_0)}}{\alpha\psi_B^2\sqrt{n}}\le \frac{29 \gamma^{-k_h}  \ln(12/\delta_0)}{\alpha\psi_B \sqrt{n}} 
     \label{eq:proof-lplr-update-8half}
    \end{align}
with probability $1-\frac{\delta_0}{6}$, provided that 
$\psi_B \ge \frac{18d}{\alpha\sqrt{n}} \ln(\frac{24d}{\delta_0})$ and $n \geq 2 \ln(\frac{12}{\delta_0})$.
For the second term on the right-hand side of Eq.~(\ref{eq:proof-lplr-update-4}), note that $\mathbb E ({\hat{\lambda}_B}/{ n})= \mathbb E \left[ \vct 1(\phi \neq \perp) y_h\phi_h\right]=\lambda_B\psi_B$,
and therefore $\hat\lambda/n$ is an unbiased estimator of$\lambda_B\psi_B$. Denote $\tilde{\lambda}_B:=n^{-1}\sum_{i=1}^n \vct 1(\phi_i\in B) y_{ih} \phi_{ih}$.
Eq.~(\ref{eq:proof-lplr-update-4half}) then implies that
\begin{align}
     \Pr \left[ \left\|\frac{\hat{\lambda}_B}{ n}-\tilde{\lambda}_B \right \|_2 \ge \frac{18d \gamma^{-k_h}}{\alpha\sqrt{n}}\ln\left(\frac{24d}{\delta_0}\right) \right] \le \frac{\delta_0}{12}.
    \label{eq:proof-lplr-update-9}
\end{align}
As for $\|\tilde{\lambda}_B- \lambda_B \psi_B\|_2$, since it's a zero-mean bounded random vector, for each coordinate we can use Hoeffding inequality, then apply union bound for all $d$ coordinate.
This yields
\begin{align}
    \Pr \left[\max_{j=1, \cdots, n}|(\tilde{\lambda}_B- \lambda_B \psi_B)_j|\ge \gamma^{-k_h} \sqrt{\frac{\ln(24d/\delta_0)}{2n}}\right] \le \frac{\delta_0}{12}.\label{eq:proof-lplr-update-10}
    \end{align}
    Subsequently,
    \begin{align}
     \Pr \left[\|(\tilde{\lambda}_B- \lambda_B \psi_B)\|_2\ge \sqrt{d}\gamma^{-k_h} \sqrt{\frac{\ln(24d/\delta_0)}{2n}}\right] \le \frac{\delta_0}{12}.\label{eq:proof-lplr-update-11}
\end{align}
Combing Eqs.~(\ref{eq:proof-lplr-update-8half},\ref{eq:proof-lplr-update-9},\ref{eq:proof-lplr-update-11}), we obtain
\begin{align}
    \mathbb P \left[ \left\|\frac{\hat{\lambda}_B}{ n\hat{\psi}_B}-\lambda_B \right \|_2 \ge \frac{(19d+29) \gamma^{-k_h}  \ln({24d}/{\delta_0})}{\alpha\psi_B \sqrt{n}} \right] \le \frac{\delta_0}{3}.\label{eq:proof-lplr-update-4}
\end{align}
provided that $\psi_B \ge \frac{18d}{\alpha\sqrt{n}} \ln(\frac{24d}{\delta_0})$ and $n \geq 2 \ln(\frac{12}{\delta_0})$. This completes the proof of the concentration inequality upper bound
involving $\hat\lambda_B$ in Lemma \ref{lem:lplr-update}, noting that $\delta_0=\delta/(2\beta^d)$.

Finally, we prove concentration results involving $\hat{\Lambda}_B$.
To facilitate the proof we introduce an equivalent formulation of Wishart random matrices:
consider $d+1$ i.i.d.~random vectors $X_j$, $j=1, \cdots, d+1$, where $X_j \sim N_d(\vct 0, 1.5\alpha^{-1} I)$, and denote $X=\left(X_1, \cdots, X_{d+1}\right) \in R^{d \times (d+1)}$,
then $XX^{\top}-\mathbb E \left[X X^{\top}\right] \sim W_d(d+1, 1.5\alpha^{-1} I)- 1.5(d+1)\alpha^{-1} I$.
Therefore, equivalently, at each time t, we generate such matrix $X_t$ and then update $\hat{\Lambda}_B$ as
\begin{align*}    \hat\Lambda_B\gets\hat\Lambda_B + \vct 1\{\phi\neq\perp\}\phi_h\phi_h^\top + 3\gamma^{-2k_h}(X_t X_t^{T}-1.5(d+1)\alpha^{-1}I)
\end{align*}

Decompose the error as
 \begin{align}
 \left\|\frac{\hat\Lambda_B}{\hat\psi_B n} - \Lambda_B\right\|_2 \;\;=    \left\|\frac{\hat\Lambda_B}{\hat\psi_B n} -\frac{\hat\Lambda_B}{\psi_B n}\right\|_2+\left\|\frac{\hat\Lambda_B}{\psi_B n} - \Lambda_B\right\|_2.
 \label{eq:proof-lplr-update-12}
 \end{align}
     For the first term, note that $\hat{\Lambda}_B=\sum_{i=1}^n \vct 1 \{\phi_i \in B\} \phi_{ ih} \phi_{ ih}^\top +3 \gamma^{-2k_h}\sum_{i=1}^n X_i$, where $X_i \sim W_d(d+1, 1.5\alpha^{-1}I)-1.5(d+1) \alpha^{-1} I$.
    For the $\sum_{i=1}^n\vct 1\{\phi_i\in B\}\phi_{ih}\phi_{ih}^\top$ term, it holds that
    \begin{align}
        \left\|\sum_{i=1}^n \vct 1 \{\phi_i \in B\} \phi_{ ih} \phi_{ ih}^\top \right\|_\op\le \gamma^{-2k_h} \sum_{i=1}^n \vct 1 \{\phi_i \in B\} \le 2 \gamma^{-2k_h}\psi_B n
        \label{eq:proof-lplr-update-13}
    \end{align}
with at least probability $1-e^{-2n}$, where the first inequality comes from $\|\phi_h \phi_h^\top\|_{\op} \leq \|\phi_h\|_2^2\le \gamma^{-2k_h}$, and the second inequality holds by 
Hoeffding's inequality. With the condition that $n \ge 2 \ln({12}/{\delta_0})$, the probability that Eq.~(\ref{eq:proof-lplr-update-13}) holds is at least $1-{\delta_0}/{12}$.
for the $\|\gamma^{-2k_h}\sum_{i=1}^n X_i \|_{\op}$ term, we 
 use concentration results of Wishart distribution (such as Lemma 1 in \cite{jiang2016wishart}).
 Specifically, in our case $C=1.5 \alpha^{-1}(d+1) I$, $\left\| C\right\|=1.5 \alpha^{-1}(d+1)$ and  $r=\text{tr}\left(C\right)/ \|C\|=d$.
Subsequently, for any $\epsilon>0$,
\begin{align}
    {\Pr}\left\{\left\|\frac{1}{n}\sum_{t=1}^n X_t X_t^{\top}-1.5 \alpha^{-1}(d+1)I\right\|_{\op} \geq 
    1.5 \alpha^{-1}(d+1)
\left(\sqrt{\frac{2\epsilon(d+1)}{n}}+ \frac{2\epsilon d}{n}\right)  \right\} \leq d \exp \left(-\epsilon\right).
\end{align}
Set $\epsilon=\ln({12d}/{\delta_0})$; we then have 
\begin{align}
    {\Pr}\left\{\left\|\frac{1}{n}\sum_{t=1}^n X_t X_t^{\top}-1.5 \alpha^{-1}(d+1)I\right\|_{\op} \geq 
    \frac{6 (d+1)^2\ln({12d}/{\delta_0})}{\alpha\sqrt{n}} \right\} \leq \frac{\delta_0}{12}.\label{eq:proof-lplr-update-14}
\end{align}
Combining Eqs.~(\ref{eq:proof-lplr-update-13},\ref{eq:proof-lplr-update-14}), we have
\begin{align}
  \frac{\|\hat{\Lambda}_B\|_{\op} }{n}\le 2 \gamma^{-2k_h}\psi_B  +\frac{18 \gamma^{-2k_h}  (d+1)^2
\ln({12d}/{\delta_0})}{\alpha\sqrt{n}}
\label{eq:proof-lplr-update-15}
\end{align}
with probability at least $1-{\delta_0}/{6}$. With the condition that $\psi_B \ge 18 (d+1)^2
\ln({12d}{\delta_0})/({\alpha\sqrt{n}})$, Eq.~(\ref{eq:proof-lplr-update-15}) is reduced to
\begin{align}
 \frac{\|\hat{\Lambda}_B\|_{\op} }{n} \leq 3\gamma^{-2k_h}\psi_B.
 \label{eq:proof-lplr-update-16}
\end{align}
This together with Eq.~(\ref{eq:proof-lplr-update-8}) yields
     \begin{align}       \left\|\frac{\hat\Lambda_B}{\hat\psi_B n} -\frac{\hat\Lambda_B}{\psi_B n}\right\|_{\op} &=\left\|\frac{\hat{\Lambda}_B}{n}\right\|_{\op} \left|\frac{1}{\hat{\psi_B}}-\frac{1}{\psi_B}\right|\le 3\gamma^{-2k_h}\psi_B\frac{{6.2\ln(12/\delta_0)}}{0.65\alpha\psi_B^2\sqrt{n}}\le \frac{29 \gamma^{-2k_h}  \ln(12/\delta_0)}{\alpha\psi_B \sqrt{n}} 
     \label{eq:proof-lplr-update-16half}
    \end{align}
with probability $1-\frac{\delta_0}{6}$, provided that
$\psi_B \ge 18 (d+1)^2
\ln({12d}/{\delta_0})/({\alpha\sqrt{n}})$ and $n \ge 2 \ln({12}/{\delta_0})$.

To upper bound the second term on the right-hand side of Eq.~(\ref{eq:proof-lplr-update-12}),
note that $\mathbb E ({\hat{\Lambda}_B}/{ n})= \mathbb E \left[ \vct 1(\phi \neq \perp) \phi_h\phi_h^T\right]=\Lambda_B\psi_B$ and therefore $\hat\Lambda_B/n$ is an unbiased estimator.
Denote $\tilde{\Lambda}_B=n^{-1}\sum_{i=1}^n \vct 1(\phi_i\in B) \phi_{ih} \phi_{ih}^T$,
 so that $\frac{\hat{\Lambda}_B}{ n}-\tilde{\Lambda}_B=\frac{1}{n}\sum_{t=1}^n X_t X_t^{T}-1.5 \alpha^{-1}(d+1)I$.
By Eq.~(\ref{eq:proof-lplr-update-14}), it holds that 
\begin{align}
    {\Pr}\left\{\left\|\frac{\hat{\Lambda}_B}{ n}-\tilde{\Lambda}_B\right\|_{\op} \geq   
\frac{ 6(d+1)^2\ln({12d}/{\delta_0})}{\alpha\sqrt{n}}  \right\} \leq \frac{\delta_0}{12}.
\label{eq:proof-lplr-update-17}
\end{align}

The last part uses matrix concentration with bounded operator norm. Specifically, using matrix Hoeffding inequality, we have for any $\epsilon>0$ that
\begin{align}
    {\Pr}\left\{\left\|\frac{1}{n}\sum_{i=1}^n\vct 1(\phi_i\in B) \gamma^{2k_h}\phi_{ih} \phi_{ih}^\top-\gamma^{2k_h} \Lambda_B\psi_B\right\|_{\op}\geq \epsilon\right\} \leq 2 d \exp \left\{-\frac{n \epsilon^2}{2 }\right\},
\end{align}
or equivalently, 
\begin{align}
    {\Pr}\left\{\left\|\frac{1}{n}\sum_{i=1}^n\vct 1(\phi_i\in B) \phi_{ih} \phi_{ih}^\top-\Lambda_B \psi_B\right\|_{\op}\geq \gamma^{-2k_h}\sqrt{\frac{2}{n}\ln({24d}/{\delta_0})}\right\} \leq \frac{\delta_0}{12},
    \label{eq:proof-lplr-update-18}
\end{align}
Combining Eqs.~(\ref{eq:proof-lplr-update-16half},\ref{eq:proof-lplr-update-16},\ref{eq:proof-lplr-update-17}), we obtain
\begin{align}
     \Pr \left[ \left\|\frac{\hat{\Lambda}_B}{ n\hat{\psi}_B}-\Lambda_B \right \|_{\op} \ge \frac{(7(d+1)^2+29) \gamma^{-2k_h}  \ln({24d}/{\delta_0})}{\alpha\psi_B \sqrt{n}} \right] \le \frac{\delta_0}{3},
\label{eq:proof-lplr-update-5}
\end{align}
provided that $\psi_B \ge 18 (d+1)^2
\ln({12d}/{\delta_0})/({\alpha\sqrt{n}})$ and $n \ge 2 \ln({12}/{\delta_0})$.
Lemma \ref{lem:lplr-update} is thus proved, by noting that $\delta_0=\delta/(2\beta^d)$.


\subsection{Proof of Lemma \ref{lem:lplr-pcr}}

Throughout this proof we operate jointly with the success events in Lemma \ref{lem:lplr-update} such that all three inequalities hold for every bin $B$ on level $h$.
For notational simplicity, for a bin $B$ we write $\mathbb E_B[\cdot] = \mathbb E_{\phi\sim\mD}[\cdot|\phi\in B]$.

Fix a particular bin $B=B_{k_1\cdots k_h}$ such that $\hat\psi_B\geq \kappa_1\gamma^2d/\sqrt{n}$. Lemma \ref{lem:lplr-update} together with the choice of parameter $\kappa_1$ imply that
\begin{align}
\psi_B\geq \frac{36(d+1)^3  \gamma^2 \ln(48d/\beta\delta)}{\alpha\sqrt{n}}.
\label{eq:proof-pcr-1}
\end{align} 
Recall the definition that
$
\Lambda_B = \mathbb E_{B}[\phi_h\phi_h^\top],
$
where $\phi_h\in\mathbb R^d$ is obtained through the procedure in Figure \ref{fig:phih} and can also be written as $\phi_h = (I-\hat U_{B^o}\hat U_{B^o}^\top)\phi$.
It is easy to see that $\Lambda_B\in\mathbb S_d^+$ and $\Lambda_B\hat U_{B^o} = 0$.
Subsequently, Line \ref{line:proj-pcr} in Algorithm \ref{alg:lplr-pcr} together with Lemma \ref{lem:lplr-update} and Eq.~(\ref{eq:proof-pcr-1}) yield
\begin{align}
\|\tilde\Lambda_B-\Lambda_B\|_\op  \le \frac{36(d+1)^3\gamma^{-2k_h}  \ln(48d/\beta\delta)}{\alpha\psi_B \sqrt{n}}  .
\label{eq:proof-pcr-2}
\end{align}
Let $s_B,\hat s_B$ be the largest eigenvalues of $\Lambda_B$ and $\tilde\Lambda_B$, respectively. By Weyl's inequality, it holds that
\begin{align}
\big|\hat s_B-s_B\big| \leq \|\tilde\Lambda_B-\Lambda_B\|_\op \leq \frac{36(d+1)^3\gamma^{-2k_h}  \ln(48d/\beta\delta)}{\alpha\psi_B \sqrt{n}}\leq \frac{\gamma^{-2k_h-2}}{2d}, \;\;\;\;\;\;\forall i\in[d],
\label{eq:proof-pcr-3}
\end{align}
where the last inequality holds because $\psi_B$ is sufficiently large as shown in Eq.~(\ref{eq:proof-pcr-1}).
Note that for every $\phi\in B$, by definition $\gamma^{-k_h-1}<\|\phi_h\|_2\leq\gamma^{-k}$. This implies
\begin{align*}
\gamma^{-2k_h-2}\leq \mathbb E_B[\|\phi_h\|_2^2] = \mathbb E_B[\tr(\phi_h\phi_h^\top)] = \tr\left(\mathbb E_B[\phi_h\phi_h^\top]\right) = \tr(\Lambda_B) \leq ds_B.
\end{align*}
Consequently, it holds that
\begin{align}
s_B\geq \frac{\gamma^{-2k_h-2}}{d}.
\label{eq:proof-pcr-4}
\end{align}
Eqs.~(\ref{eq:proof-pcr-3},\ref{eq:proof-pcr-4}) together imply that
\begin{align}
\hat s_B\geq \frac{\gamma^{-2k_h-2}}{2d}.
\label{eq:proof-pcr-4half}
\end{align}
This proves the upper bound on $\hat s_B^{-1}$ in Lemma \ref{lem:lplr-pcr}.

Decompose the estimation error $\hat\theta_B-\theta^*$ as
\begin{align}
\hat\theta_B-\theta_B^*&=\hat u_B\hat s_B^{-1}\hat u_B^\top\frac{\hat\lambda_B}{\hat\psi_B n} - \hat u_B\hat u_B^\top\theta^*\nonumber\\
&= \hat u_B\hat s_B^{-1}\hat u_B^\top\lambda_B - \hat u_B\hat u_B^\top\theta^* + \hat u_B\hat s_B^{-1}\hat u_B^\top\left(\frac{\hat\lambda_B}{\hat\psi_B n}-\lambda_B\right) .\label{eq:proof-pcr-9}
\end{align}
Recall the definitions that $\lambda_B = \mathbb E_{B}[y_h\phi_h]$, $\phi_h = (I-\hat U_{B^o}\hat U_{B^o}^\top)\phi$ and $y_h=y-\phi^\top\hat\theta_{B^o}^\tot$. 
Note also that $\theta_B^\tot = \sum_{\ell=1}^h\hat u_{B_{k_1\cdots k_\ell}}\hat u_{B_{k_1\cdots k_\ell}}^\top\theta^* = \hat U_B\hat U_B^\top\theta^*$.
Let $y = \phi^\top\theta^* + \xi$ such that $\mathbb E[\xi|\phi]=0$.
We then have
\begin{align}
\lambda_B &= \mathbb E_B[y_h\phi_h] = \mathbb E_B\left[(y-\phi^\top\hat\theta_{B^o}^\tot)(I-\hat U_{B^o}\hat U_{B^o}^\top)\phi\right]\nonumber\\
&= \mathbb E_B\left[\phi^\top(\theta_{B^o}^\tot-\hat\theta_{B^o}^\tot)(I-\hat U_{B^o}\hat U_{B^o}^\top)\phi\right] 
+ \mathbb E_B\left[\phi^\top(I-\hat U_{B^o}\hat U_{B^o}^\top)\theta^*(I-\hat U_{B^o}\hat U_{B^o}^\top)\phi\right]\nonumber\\
&= -\mathbb E_B\left[\eta_{B^o}(\phi)\phi_h\right] + \mathbb E_{B}[\phi_h^\top\theta^*\phi_h]=  -\mathbb E_B\left[\eta_{B^o}(\phi)\phi_h\right]  + \Lambda_B\theta^*.\label{eq:proof-pcr-10}
\end{align}
Comparing Eqs.~(\ref{eq:proof-pcr-9},\ref{eq:proof-pcr-10}), we obtain
\begin{align}
&\hat\theta_B-\theta_B^* = -\hat u_B\hat s_B^{-1}\hat u_B^\top\mathbb E_B\left[\eta_{B^o}(\phi)\phi_h\right] + \hat u_B\hat s_B^{-1}\hat u_B^\top\Lambda_B\theta^*  - \hat u_B\hat u_B^\top\theta^* + \hat u_B\hat s_B^{-1}\hat u_B^\top\left(\frac{\hat\lambda_B}{\hat\psi_B n}-\lambda_B\right)\nonumber\\
&= -\hat u_B\hat s_B^{-1}\hat u_B^\top\mathbb E_B\left[\eta_{B^o}(\phi)\phi_h\right] + \hat u_B\hat s_B^{-1}\hat u_B^\top(\Lambda_B-\tilde\Lambda_B)\theta^*+ \hat u_B\hat s_B^{-1}\hat u_B^\top\left(\frac{\hat\lambda_B}{\hat\psi_B n}-\lambda_B\right),\label{eq:proof-pcr-11}
\end{align}
where the last inequality holds because $\hat u_B\hat s_B^{-1}\hat u_B^\top\tilde\Lambda_B = \hat u_B\hat u_B^\top$.
Define
$$
\omega_B := \hat u_B\hat s_B^{-1}\hat u_B^\top(\Lambda_B-\tilde\Lambda_B)\theta^*+ \hat u_B\hat s_B^{-1}\hat u_B^\top\left(\frac{\hat\lambda_B}{\hat\psi_B n}-\lambda_B\right).
$$
Consequently, 
\begin{align}
\|\omega_B\|_2 &\leq \hat s_B^{-1}\left(\|\tilde\Lambda_B-\Lambda_B\|_\op + \left\|\frac{\hat\lambda_B}{\hat\psi_B n}-\lambda_B\right\|_2\right)\nonumber\\
&\leq 2d\gamma^{2k_h+2}\left(\frac{36(d+1)^3\gamma^{-2k_h}  \ln(48d/\beta\delta)}{\alpha\psi_B \sqrt{n}} + \frac{48d^2 \gamma^{-k_h}  \ln(d/\beta\delta)}{\alpha\psi_B \sqrt{n}}\right)\label{eq:proof-pcr-12}\\
&\leq \frac{120(d+1)^4 \gamma^{k_h+2}  \ln(48d/\beta\delta)}{\alpha\psi_B \sqrt{n}}.
\end{align}
Here, Eq.~(\ref{eq:proof-pcr-12}) holds by incorporating Eqs.~(\ref{eq:proof-pcr-2},\ref{eq:proof-pcr-4half}) and Lemma \ref{lem:lplr-update}.
This proves the upper bound on $\|\omega_B\|_2$ in Lemma \ref{lem:lplr-pcr}.

\subsection{Proof of Lemma \ref{lem:lplr-ci}}

We first show that Algorithm \ref{alg:lplr-ci} satisfies $\alpha$ local differential privacy. This is trivially proved by noting that $|\varepsilon_i|\leq \gamma^{-k_h}\hat s_B^{-1/2}$ almost surely
because $\hat\eta(\cdot)\in[0,1]$ and $\|\hat u_B^\top\phi_i\|_2\leq \|\phi_{ih}\|_2\leq \gamma^{-k_h}$. The proof is then essentially the same as the proof of Proposition \ref{prop:ldp}.

To prove the rest of the lemma we will state and prove some support lemmas, which also motivate the construction of $\hat\varepsilon_B$ as well as the confidence intervals
in Algorithm \ref{alg:lplr-ci}. Throughout this section we will work jointly with the success event that all inequalities in Lemmas \ref{lem:lplr-update} and \ref{lem:lplr-pcr} hold for all bins on level $h$,
and that the conclusion in Lemma \ref{lem:lplr-ci} holds for all previous layers $1,2,\cdots,h-1$. All high-probability events and expectations are also conditioned on previous levels
and the first batch of $n$ samples on level $h$, so that $\{\hat u_B\}$ for bin $B$ on level $h$ are conditioned upon and do not interfere with randomness in data received by Algorithm \ref{alg:lplr-ci}.
For every bin $B=B_{k_1\cdots k_h}$ on level $h$, define
\begin{align}
\varepsilon_B &:= \mathbb E_B\left[\hat\eta_{B^o}(\phi) \big|\hat s_B^{-1/2}\hat u_B^\top\phi\big|\right],
\end{align}
where $\mathbb E_B[\cdot] = \mathbb E_{\phi\sim\mD}[\cdot|\phi\in B]$. We then have the following lemma:
\begin{lemma}
Suppose $n \ge 2d\ln(24d/\beta\delta)$. Then with probability $1-\delta$ the following holds uniformly for all bin $B=B_{k_1\cdots k_h}$:
$$
\left|\frac{\hat\varepsilon_B}{\hat\psi_B n} - \varepsilon_B\right|\leq \frac{47 d^{3/2}\gamma\ln(24d/\beta\delta)}{\alpha\psi_B \sqrt{n}}, \;\;\;\;\;\;\; \text{if } \psi_B \ge \frac{8.5d^{3/2}\gamma\ln(12d/\beta\delta)}{\alpha\sqrt{n}}.
$$
\label{lem:epsi-concentration}
\end{lemma}
\begin{proof}{Proof of Lemma \ref{lem:epsi-concentration}.}
For simplicity we fix $\delta_0\in(0,1]$ and a particular bin $B=B_{k_1\cdots k_h}$. By union bound, lemma \ref{lem:epsi-concentration} holds uniformly over all bins $B$ except for a failure probability of $\delta=M^h\delta_0\leq (2\beta)^{-d}\delta_0$. 

Jointly with the success event in Lemma \ref{lem:lplr-update}, it holds that 
\begin{align}
\big|\hat\psi_B-\psi_B\big| \leq \frac{6.2\ln(12/\delta_0)}{\alpha\sqrt{n}}.\label{eq:proof-epsic-1}
\end{align}
Decompose the estimation error of $\varepsilon_B$ as follows:
 \begin{align}
 \left\|\frac{\hat\varepsilon_B}{\hat\psi_B n} - \varepsilon_B\right\|_2 \leq     \left\|\frac{\hat\varepsilon_B}{\hat\psi_B n} -\frac{\hat\varepsilon_B}{\psi_B n}\right\|_2+\left\|\frac{\hat\varepsilon_B}{\psi_B n} - \varepsilon_B\right\|_2.
 \label{eq:proof-epsic-2}
 \end{align}
    
  We first upper bound the first term on the right-hand side of Eq.~(\ref{eq:proof-epsic-2}). 
  To do this, we shall first upper bound $|\hat\varepsilon_B|$ with high probability. Note that
   $\hat{\varepsilon}_B=\sum_{i=1}^n \vct 1\{\phi_i \in B\}\varepsilon_i + X_i$, where $X_i \sim  \hat s_{B}^{-1/2}\gamma^{-k_h} \lap_1(0,1/\alpha)$.
  By Hoeffding's inequality and the observation that $|\varepsilon_i|=|\hat\eta_{B^o}(\phi)|\hat u_B^\top\phi_i|/\sqrt{\hat s_B}|\le \hat s_{B}^{-1/2}\gamma^{-k_h}$,
  we have
    \begin{align}
       \left|\sum_{i=1}^n \vct 1 \{\phi_i \in B\}\varepsilon_i \right|\le\hat s_{B}^{-1/2}\gamma^{-k_h} \sum_{i=1}^n \vct 1 \{\phi_i \in B\} \le 2 \hat s_{B}^{-1/2}\gamma^{-k_h}\psi_B n
       \label{eq:proof-epsic-3}
    \end{align}
with at least probability $1-e^{-2n}$. With the condition that $n \ge 2 \ln({4}/{\delta_0})$, the probability that the above inequality holds is at least $1-{\delta_0}/{4}$.
For the $|\sum_{i=1}^n X_i |$ term, note that it is a centered sub-exponential random variable with $\nu=2 \sqrt{n}\alpha^{-1}$ and $b=2\alpha^{-1}$. Using  Bernstein inequality, we have
\begin{align}
     \Pr\left[ \left|\sum_{i=1}^n X_i\right| \ge  6 \alpha^{-1}{\sqrt{n}}\ln({6}/{\delta_0})|\right] \le \frac{\delta_0}{4},
     \label{eq:proof-epsic-4}
\end{align}
which implies with probability $1-\delta_0/4$ that
\begin{align}
   \left|\hat s_{B}^{-1/2}\gamma^{-k_h}\sum_{i=1}^n X_i \right|
    \le 6\hat s_{B}^{-1/2}\gamma^{-k_h}\alpha^{-1}{\sqrt{n}}\ln({6}/{\delta_0}).
    \label{eq:proof-epsic-5}
\end{align}
Combining Eqs.~(\ref{eq:proof-epsic-3},\ref{eq:proof-epsic-5}) and using the union bound, we have with probability $1-\delta_0/2$ that
\begin{align}
\big|\hat{\varepsilon}_B\big| \le 2\hat s_{B}^{-1/2}\gamma^{-k_h}\psi_B n +6\hat s_{B}^{-1/2}\gamma^{-k_h}\alpha^{-1}{\sqrt{n}}\ln({6}/{\delta_0}).
\label{eq:proof-epsic-6}
\end{align}
With the condition that $\psi_B \ge 6\hat s_{B}^{-1/2}\gamma^{-k_h}\alpha^{-1}\ln({6}/{\delta_0})/\sqrt{n}$, Eq.~(\ref{eq:proof-epsic-6}) is further reduced to 
\begin{align}
\big|\hat{\varepsilon}_B\big| \le 3\hat s_{B}^{-1/2}\gamma^{-k_h}\psi_B n.
\label{eq:proof-epsic-7}
\end{align} 
Note additionally that $\hat s_B\geq \gamma^{-2k_h-2}/(2d)$ thanks to Eq.~(\ref{eq:proof-pcr-4half}) in the proof of Lemma \ref{lem:lplr-pcr}.
Consequently, the condition $\psi_B \ge 6\hat s_{B}^{-1/2}\gamma^{-k_h}\alpha^{-1}\ln({6}/{\delta_0})/\sqrt{n}$ is implied by 
$$
\psi_B\geq \frac{8.5\gamma \sqrt{d}\ln(6/\delta_0)}{\alpha\sqrt{n}}.
$$

With the condition on $\psi_B$ and the upper bound on $|\hat\psi_B-\psi_B|$ from Lemma \ref{lem:lplr-update}, 
and noting that $\gamma\geq 2$, it holds that $\hat\psi_B\psi_B\geq 0.74\psi_B^2$. Subsequently,
\begin{align}
    \left|\frac{1}{\hat{\psi}_B}-\frac{1}{\psi_B}\right| \le \frac{6.2\ln(12/\delta_0)}{0.74\psi_B^2\alpha\sqrt{n}} \leq \frac{8.4\ln(12/\delta_0)}{\psi_B^2\alpha\sqrt{n}}.
    \label{eq:proof-epsic-8}
\end{align}
Combining Eqs.~(\ref{eq:proof-epsic-7},\ref{eq:proof-epsic-8}) and noting that $\hat s_B\geq \gamma^{-2k_h-2}/(2d)$, we obtain
     \begin{align}       \left\|\frac{\hat\varepsilon_B}{\hat\psi_B n} -\frac{\hat\varepsilon_B}{\psi_B n}\right\|_2 &\leq \left|\frac{\hat{\varepsilon}_B}{n}\right| \left|\frac{1}{\hat{\psi}_B}-\frac{1}{\psi_B}\right|
    \le \frac{3}{n}\left(\hat{s}_{B}^{-1/2}\gamma^{-k_h}\psi_B n\right)   \frac{8.4\ln(12/\delta_0)}{\psi_B^2\alpha\sqrt{n}}\nonumber\\
    &\le \frac{26 \hat{s}_{B}^{-1/2}\gamma^{-k_h}  \ln(12/\delta_0)}{\alpha\psi_B \sqrt{n}}
    \leq \frac{37\gamma\sqrt{d}\ln(12/\delta_0)}{\alpha\psi_B\sqrt{n}} \label{eq:proof-epsic-9}
    \end{align}
with probability $1-{\delta_0}/{2}$, provided that 
$\psi_B\geq \frac{8.5\gamma \sqrt{d}\ln(6/\delta_0)}{\alpha\sqrt{n}}$ and $n \ge 2 \ln({4}/{\delta_0})$.

We next upper bound the second term on the right-hand side of Eq.~(\ref{eq:proof-epsic-2}). Note that $\mathbb E (\frac{\hat{\varepsilon}_B}{ n})= \mathbb E \left[ \vct 1(\phi \neq \perp) \varepsilon\right]=\varepsilon_B\psi_B$, and therefore $\hat\varepsilon_B/n$ is an unbiased estimator. Denote $\tilde{\varepsilon}_B=n^{-1}\sum_{i=1}^n \vct 1(\phi_i\in B) \varepsilon_i$.
Eq.~(\ref{eq:proof-epsic-3}) then implies 
\begin{align}
     \Pr \left[ \left|\frac{\hat{\varepsilon}_B}{ n}-\tilde{\varepsilon}_B \right| \ge \frac{6\hat s_{B}^{-1/2}\gamma^{-k_h}\ln({6}/{\delta_0})}{\alpha\sqrt{n}}\right] \le \frac{\delta_0}{4}.
     \label{eq:proof-epsic-10}
\end{align}
For the $|\tilde{\varepsilon}_B- \varepsilon_B \psi_B|$, because it is a sum of centered bounded and independent random variables, Hoeffding's inequality yields  
\begin{align}
     \Pr \left[|(\tilde{\varepsilon}_B- \varepsilon_B \psi_B)|\ge \hat{s}_{B}^{-1/2}\gamma^{-k_h} \sqrt{\frac{\ln(8/\delta_0)}{2n}}\right] \le \frac{\delta_0}{4}.
     \label{eq:proof-epsic-11}
\end{align}
Note that $\hat s_B\geq \gamma^{-2k_h-2}/(2d)$ thanks to Eq.~(\ref{eq:proof-pcr-4half}). Eqs.~(\ref{eq:proof-epsic-10},\ref{eq:proof-epsic-11}) then yield
\begin{align}
&\Pr\left[ \left|\frac{\hat{\varepsilon}_B}{ n}-\tilde{\varepsilon}_B \right|\geq \frac{6\sqrt{2d}\gamma\ln(6/\delta_0)}{\alpha\sqrt{n}}\right] \leq \frac{\delta_0}{4};\label{eq:proof-epsic-10half}\\
&\Pr \left[|(\tilde{\varepsilon}_B- \varepsilon_B \psi_B)| \ge \sqrt{2d}\gamma \sqrt{\frac{\ln(8/\delta_0)}{2n}}\right] \le \frac{\delta_0}{4}.
\label{eq:proof-epsic-11half}
\end{align}
Combining Eqs.~(\ref{eq:proof-epsic-9},\ref{eq:proof-epsic-10half},\ref{eq:proof-epsic-11half}), we have that 
\begin{align}
     \Pr \left[ \left|\frac{\hat{\varepsilon}_B}{ n\hat{\psi}_B}-\varepsilon_B \right|\ge \frac{47 \gamma\sqrt{d} \ln({12}/{\delta_0})}{\alpha\psi_B \sqrt{n}} \right] \le \delta_0.
\end{align}
Lemma \ref{lem:epsi-concentration} is then proved, by plugging in $\delta_0=\delta/(2\beta^d)$. $\square$
\end{proof}

\begin{corollary}
Suppose $n\geq 2d\ln(24d/\beta\delta)$ and $\kappa_2$ is chosen as 
$$
\kappa_2\geq 118\alpha^{-1}\ln(24d/\beta\delta).
$$
The inequality in Lemma \ref{lem:epsi-concentration} jointly with the inequalities in Lemma \ref{lem:lplr-update} that 
$$
\varepsilon_B\leq\bar\varepsilon_B\leq\varepsilon_B + \frac{2.5\kappa_2\gamma d^{3/2}}{\psi_B \sqrt{n}}
$$
for all bins $B=B_{k_1\cdots k_h}$ such that $\hat\psi_B\geq \frac{15d^{3/2}\gamma\ln(48d/\beta\delta)}{\alpha\sqrt{n}}$.
\label{cor:bareps}
\end{corollary}
\begin{proof}{Proof of Corollary \ref{cor:bareps}.}
The upper bound on $|\hat\psi_B-\psi_B|$ from Lemma \ref{lem:lplr-update} implies that, if 
$$
\hat\psi_B\geq \frac{15d^{3/2}\gamma\ln(48d/\beta\delta)}{\alpha\sqrt{n}}
$$
then the condition in Lemma \ref{lem:epsi-concentration}
$$
\psi_B  \ge \frac{8.5d^{3/2}\gamma\ln(12d/\beta\delta)}{\alpha\sqrt{n}}
$$
holds, which leads to all inequalities in Lemma \ref{lem:epsi-concentration} with high probability.
This condition (lower bound on $\hat\psi_B$) also implies via Lemma \ref{lem:lplr-update} that
$$
0.4\psi_B\leq \hat\psi_B \leq 2.5\psi_B
$$
With the condition on $\kappa_2$ and the above sandwiching inequality on $\hat\psi_B$, Lemma \ref{lem:epsi-concentration} implies that 
\begin{align*}
\left|\frac{\hat\varepsilon_B}{\hat\psi_B n}-\varepsilon_B\right| &\leq \frac{\kappa_2\gamma d^{3/2}}{2.5\psi_B\sqrt{n}}  \leq \frac{\kappa_2\gamma d^{3/2}}{\hat\psi_B\sqrt{n}}\leq \frac{2.5\kappa_2\gamma d^{3/2}}{\psi_B\sqrt{n}},
\end{align*}
which proves Corollary \ref{cor:bareps}. $\square$
\end{proof}

\begin{lemma}
For every active bin $B=B_{k_1\cdots k_h}$ and $\phi\in B$, 
$$
\big|\phi^\top(\hat\theta_B-\theta_B^*)\big| \leq \varepsilon_B\big|\hat s_B^{-1/2}\hat u_B^\top\phi\big| + \frac{120(d+1)^4 \gamma^{2}  \ln(48d/\beta\delta)}{\alpha\psi_B \sqrt{n}}.
$$
\label{lem:epsi-ub}
\end{lemma}
\begin{proof}{Proof of Lemma \ref{lem:epsi-ub}.}
We start with the error decomposition established in Lemma \ref{lem:lplr-pcr}.
Because both $\hat\theta_B$ and $\theta_B^*$ are with the same direction of $\hat u_B$, it holds that $\phi^\top(\hat\theta_B-\theta_B^*) = \phi_h^\top(\hat\theta_B-\theta_B^*)$
where $\phi_h=(I-\hat U_{B^o}\hat U_{B^o}^\top)\phi$ is the vector calculated by Figure \ref{fig:phih} as input to level $h$.
Subsequently, $|\phi_h^\top\omega_B|\leq \|\phi_h\|_2\|\omega_B\|_2\leq \gamma^{-k_h}\|\omega_B\|_2$ which together with the upper bound on $\|\omega_B\|_2$
from Lemma \ref{lem:lplr-pcr} gives rises to the second term in the stated inequality.

We next analyze the first term as follows: for every $\varphi\in B$,
\begin{align}
\big|\varphi^\top&\hat u_B\hat s_B^{-1}\hat u_B^\top\mathbb E_B[\eta_{B^o}(\phi)\phi_h]\big|
\leq \big|\hat s_B^{-1/2}\hat u_B^\top\varphi\big|\times \big|\mathbb E_B[\eta_{B^o}(\phi)\hat s_B^{-1/2}\hat u_B^\top\phi_h]\big|\nonumber\\
&= \big|\hat s_B^{-1/2}\hat u_B^\top\varphi\big|\times \big|\mathbb E_B[\eta_{B^o}(\phi)\hat s_B^{-1/2}\hat u_B^\top\phi]\big|\label{eq:proof-epsi-1}\\
&\leq\big|\hat s_B^{-1/2}\hat u_B^\top\varphi\big|\times \mathbb E_B\left[\big|\eta_{B^o}(\phi)\big|\cdot \big|\hat s_B^{-1/2}\hat u_B^\top\phi\big|\right]\label{eq:proof-epsi-2}\\
&\leq\big|\hat s_B^{-1/2}\hat u_B^\top\varphi\big|\times \mathbb E_B\left[\hat\eta_{B^o}(\phi)\big|\hat s_B^{-1/2}\hat u_B^\top\phi\big|\right] \label{eq:proof-epsi-3}\\
&\leq \varepsilon_B\big|\hat s_B^{-1/2}\hat u_B^\top\varphi\big|.\nonumber
\end{align}
Here, Eq.~(\ref{eq:proof-epsi-1}) holds because $\phi_h=(I-\hat U_{B^o}\hat U_{B^o}^\top)\phi$ and $\hat u_B^\top\hat U_{B^o}=0$;
Eq.~(\ref{eq:proof-epsi-2}) holds by the convexity of $|\cdot|$;
Eq.~(\ref{eq:proof-epsi-3}) holds because $|\eta_{B^o}(\phi)|\leq\hat\eta_{B^o}\phi$ for every $\phi\in B\subseteq B^o$ as assumed Lemma \ref{lem:lplr-ci}.
This completes the proof of Lemma \ref{lem:epsi-concentration}. $\square$
\end{proof}

\begin{corollary}
Suppose the conditions on $n$ and $\kappa_2$ hold in Corollary \ref{cor:bareps}.
For every bin $B=B_{k_1\cdots k_h}$ whose $\hat\psi_B$ satisfies the condition in Corollary \ref{cor:bareps} and every $\phi\in B$,
$$
\big|\phi^\top(\hat\theta_B-\theta_B^*)\big| \leq \bar\varepsilon_B\big|\hat s_B^{-1/2}\hat u_B^\top\phi\big| + \frac{300(d+1)^4 \gamma^{2}  \ln(48d/\beta\delta)}{\alpha\hat\psi_B \sqrt{n}}.
$$
\label{cor:epsi-ub}
\end{corollary}
\begin{proof}{Proof of Corollary \ref{cor:epsi-ub}.}
The corollary immediately holds by using $\varepsilon_B\leq\bar\varepsilon_B$ and $0.4\psi_B\leq\hat\psi_B\leq 2.5\psi_B$ proved in the proof of Corollary \ref{cor:bareps}. $\square$
\end{proof}

Note the value and scaling of algorithmic parameter $\kappa_3$ in the definition of $\hat\chi_B(\cdot)$.
Lemma \ref{lem:epsi-ub} and Corollary \ref{cor:bareps} immediately establish that, for every $\phi\in B$,
\begin{align*}
\big|\chi_B(\phi)\big| &= \big|\phi^\top(\hat\theta_B-\theta_B^*)\big| \leq \hat\chi_B(\phi);\\
\big|\eta_B(\phi)\big| &= \big|\phi^\top(\hat\theta_B^\tot-\theta_B^\tot)\big| \leq \big|\phi^\top(\hat\theta_{B^o}^\tot-\theta_{B^o}^\tot)\big| + \big|\phi^\top(\hat\theta_B-\theta_B^*)\big|\\
&\leq \hat\eta_{B^o}(\phi) + \hat\chi_B(\phi) = \hat\eta_B(\phi),
\end{align*}
which prove the validity of $\hat\eta_B(\cdot)$ and $\hat\chi_B(\cdot)$.
Furthermore, for $\phi\in B$ where $B$ is inactive, it holds that
$$
\big|\phi^\top(\hat\theta_B^\top-\theta^*)\big| \leq \|\phi\|_2 \leq \gamma^{-k_h} \leq \hat\eta_B(\phi).
$$

We next focus on the recursive inequality in Lemma \ref{lem:lplr-ci}. Again the following analysis holds for active bins in level $h$.
By Cauchy-Schwarz inequality, we have that
\begin{align}
\varepsilon_B &\leq \sqrt{\mathbb E_B[\hat\eta_{B^o}(\phi)^2]}\sqrt{\hat s_B^{-1}\mathbb E_B[|\hat u_B^\top\phi|^2]}.
\label{eq:proof-epsi-4}
\end{align}
Recall that $s_B$ is the largest eigenvalue of $\Lambda_B=\mathbb E_B[\phi_h\phi_h^\top]$. Subsequently,
\begin{align}
\mathbb E_B[|\hat u_B^\top\phi|^2] = \mathbb E_B[|\hat u_B^\top\phi_h|^2] =\hat u_B^\top\Lambda_B\hat u_B\leq s_B.
\label{eq:proof-epsi-4quart}
\end{align}
Plugging the above inequality and Eqs.~(\ref{eq:proof-pcr-3},\ref{eq:proof-pcr-4},\ref{eq:proof-pcr-4half}) in the proof of Lemma \ref{lem:lplr-pcr} into Eq.~(\ref{eq:proof-epsi-4}), we obtain
\begin{align}
\varepsilon_B &\leq \sqrt{s_B/\hat s_B}\times \sqrt{\mathbb E_B[\hat\eta_{B^o}(\phi)^2]}\leq \sqrt{2}\sqrt{\mathbb E_B[\hat\eta_{B^o}(\phi)^2]}.
\label{eq:proof-epsi-4half}
\end{align}
We then have for every $\varphi\in B$ that 
\begin{align}
&\hat\eta_B(\varphi)^2 = (\hat\eta_{B^o}(\varphi)+\hat\chi_B(\varphi))^2 = \left(\hat\eta_{B^o}(\varphi)+\frac{\kappa_3\gamma^2(d+1)^4}{\hat\psi_B\sqrt{n}} + \frac{\bar\varepsilon_B|\hat u_B^\top\varphi|}{\sqrt{\hat s_B}}\right)^2\nonumber\\
&\leq \left[\hat\eta_{B^o}(\varphi)+\frac{\kappa_3\gamma^2(d+1)^4}{\hat\psi_B\sqrt{n}}  + \frac{2.5\kappa_2\gamma d^{3/2}}{\psi_B\sqrt{n}}\frac{|\hat u_B^\top\varphi|}{\sqrt{\hat s_B}} + \varepsilon_B\frac{|\hat u_B^\top\varphi|}{\sqrt{\hat s_B}}\right]^2\label{eq:proof-epsi-5}\\
&\leq 1.1\left[\hat\eta_{B^o}(\varphi)+\varepsilon_B\frac{|\hat u_B^\top\varphi|}{\sqrt{\hat s_B}}\right]^2 + 11\left[\frac{\kappa_3\gamma^2(d+1)^4}{\hat\psi_B\sqrt{n}}  + \frac{2.5\kappa_2\gamma d^{3/2}}{\psi_B\sqrt{n}}\frac{|\hat u_B^\top\varphi|}{\sqrt{\hat s_B}}\right]^2\nonumber\\
&\leq 1.1\hat\eta_{B^o}(\varphi)^2+2.2\varepsilon_B\frac{\hat\eta_{B^o}(\varphi)|\hat u_B^\top\varphi|}{\sqrt{\hat s_B}} + 1.1\frac{\varepsilon_B^2|\hat u_B^\top\varphi|^2}{\hat s_B}
+\frac{22\kappa_3^2\gamma^4(d+1)^8}{\hat\psi_B^2 n} +\frac{138\kappa_2^2\gamma^2d^3|\hat u_B^\top\varphi|^2}{\psi_B^2\hat s_B n}\nonumber\\
&\leq  1.1\hat\eta_{B^o}(\varphi)^2+2.2\varepsilon_B\frac{\hat\eta_{B^o}(\varphi)|\hat u_B^\top\varphi|}{\sqrt{\hat s_B}} + 1.1\frac{\varepsilon_B^2|\hat u_B^\top\varphi|^2}{\hat s_B}
+138\left(\frac{\kappa_3^2\gamma^4(d+1)^8}{\psi_B^2 n} +\frac{\kappa_2^2\gamma^2d^4}{\psi_B^2n}\right)\label{eq:proof-epsi-6}.
\end{align}
Here, Eq.~(\ref{eq:proof-epsi-5}) holds by invoking the upper bound on $\bar\varepsilon_B$ in Corollary \ref{cor:bareps};
Eq.~(\ref{eq:proof-epsi-6}) holds because $\hat\psi_B\geq 0.4\psi_B$ proved in the proof of Corollary \ref{cor:bareps}, and that $|\hat u_B^\top\varphi|\leq \sqrt{d}\hat s_B$
for every $\varphi\in B$.
Taking expectations with respect to $\mathbb E_B[\cdot]=\mathbb E_{\varphi\sim\mD}[\cdot|\varphi\in B]$ on both sides of Eq.~(\ref{eq:proof-epsi-6}), we obtain
\begin{align}
 \mathbb E_B[\hat\eta_B(\varphi)^2] &\leq 1.1\mathbb E_B[\hat\eta_{B^o}(\varphi)^2] + 2.2\varepsilon_B\frac{\mathbb E_B[\hat\eta_{B^o}(\varphi)|\hat u_B^\top\varphi|]}{\sqrt{\hat s_B}} + 1.1\varepsilon_B^2\frac{\mathbb E_B[|\hat u_B^\top\varphi|^2]}{\hat s_B}\nonumber\\
&\;\;\;\; + 138\left(\frac{\kappa_3^2\gamma^4(d+1)^8}{\psi_B^2 n} +\frac{\kappa_2^2\gamma^2d^4}{\psi_B^2n}\right)\nonumber\\
&\leq 1.1\mathbb E_B[\hat\eta_{B^o}(\varphi)^2] + 2.2\varepsilon_B \frac{\sqrt{\mathbb E_B[\hat\eta_{B^o}(\varphi)^2]}\sqrt{\mathbb E_B[|\hat u_B^\top\varphi|^2]}}{\sqrt{\hat s_B}} + 1.1\varepsilon_B^2\frac{\mathbb E_B[|\hat u_B^\top\varphi|^2]}{\hat s_B}\nonumber\\
&\;\;\;\; + 138\left(\frac{\kappa_3^2\gamma^4(d+1)^8}{\psi_B^2 n} +\frac{\kappa_2^2\gamma^2d^4}{\psi_B^2n}\right)\label{eq:proof-epsi-7}\\
&\leq 1.1\mathbb E_B[\hat\eta_{B^o}(\varphi)^2] + 2.2\mathbb E_B[\hat\eta_{B^o}(\varphi)^2]\frac{\sqrt{\mathbb E_B[|\hat u_B^\top\varphi|^2]}}{\sqrt{\hat s_B}} + 1.1\mathbb E_B[\hat\eta_{B^o}(\varphi)^2]\frac{\mathbb E_B[|\hat u_B^\top\varphi|^2]}{\hat s_B}\nonumber\\
&\;\;\;\;+ 138\left(\frac{\kappa_3^2\gamma^4(d+1)^8}{\psi_B^2 n} +\frac{\kappa_2^2\gamma^2d^4}{\psi_B^2n}\right)\label{eq:proof-epsi-8}\\
&\leq 6.42\mathbb E_B[\hat\eta_{B^o}(\varphi)^2]  + 138\left(\frac{\kappa_3^2\gamma^4(d+1)^8}{\psi_B^2 n} +\frac{\kappa_2^2\gamma^2d^4}{\psi_B^2n}\right).\label{eq:proof-epsi-9}
\end{align}
Here, Eq.~(\ref{eq:proof-epsi-7}) holds by Cauchy-Schwarz inequality; 
Eq.~(\ref{eq:proof-epsi-8}) holds by invoking Eq.~(\ref{eq:proof-epsi-4half});
Eq.~(\ref{eq:proof-epsi-9}) holds by invoking Eq.~(\ref{eq:proof-epsi-4quart}) and noting that $s_B/\hat s_B\leq 2$.
Note also that, for every inactive bin $B$, we have
$\hat{\eta}_B(\phi)=\hat{\eta}_{B^o}(\phi)+\gamma^{-k_h}$, so $\sqrt{\mathbb E_B [\hat{\eta}_B(\phi)^2]} \le 2\sqrt{\mathbb E_B [\hat{\eta}_{B^o}(\phi)^2]}+2 \gamma^{-k_h}$.

We now consider an arbitrary active $B=B_{k_1\cdots k_{h-1}}$ on level $h-1$ and let $B^{+0},\cdots,B^{+(M-1)}$ be the child bins of $B$ on level $h$, where $M=\lceil1/(2\beta)\rceil$.
For $\psi(B^{+k}|B) := \psi_{B^{+k}}/\psi_B$ such that $\sum_{k=0}^{M-1}\psi(B^{+k}|B)=1$.
Then
\begin{align}
&\qquad\sum_{k=0}^{M-1}\psi_{B^{+k}}\sqrt{\mathbb E_{B^{+k}}[\hat\eta_{B^{+k}}(\phi)^2]}\\
&\leq \sum_{B^{+k} \text{active}}\psi_{B^{+k}}\left(2.54\sqrt{\mathbb E_{B^{+k}}[\hat\eta_{B}(\phi)^2]} + \frac{12\kappa_3\gamma^2(d+1)^4}{\psi_B\sqrt{n}} + \frac{12\kappa_2\gamma d^2}{\psi_B\sqrt{n}}\right)\nonumber \\
& \qquad \qquad +\sum_{B^{+k} \text{inactive}} \psi_{B^{+k}}\left(\sqrt{\mathbb E [\hat{\eta}_{B^o}(\phi)^2]}+2 \gamma^{-k_h}\right)\\
&\leq \frac{12M(\kappa_3\gamma^2(d+1)^4+\kappa_2\gamma d^2+\max\{\kappa_1\gamma^2(d+1)^3,\kappa_1'\gamma d^{3/2}\})}{\sqrt{n}}\nonumber\\
&\qquad \qquad + 2.54\psi_B\sum_{k=0}^{M-1}\psi(B^{+k}|B)\sqrt{\mathbb E_{B^{+k}}[\hat\eta_{B}(\phi)^2]}\label{eq:proof-epsi-4half}\\
&\leq \frac{12M(\kappa_3\gamma^2(d+1)^4+\kappa_2\gamma d^2+\max\{\kappa_1\gamma^2(d+1)^3,\kappa_1'\gamma d^{3/2}\})}{\sqrt{n}}\nonumber\\
&\qquad\qquad +2.54\psi_B \sqrt{\sum_{k=0}^{M-1}\psi(B^{+k}|B)}\sqrt{\sum_{k=0}^{M-1}\psi(B^{+k}|B)\mathbb E_{B^{+k}}[\hat\eta_B(\phi)^2]}\label{eq:proof-epsi-5}\\
&\leq\frac{12M(\kappa_3\gamma^2(d+1)^4+\kappa_2\gamma d^2+\max\{\kappa_1\gamma^2(d+1)^3,\kappa_1'\gamma d^{3/2}\})}{\sqrt{n}}\nonumber\\
&\qquad\qquad+ 2.54\psi_B\sqrt{\mathbb E_B[\hat\eta_B(\phi)^2]}.\label{eq:proof-epsi-6}
\end{align}
Here, Eq.~(\ref{eq:proof-epsi-4half}) holds by noting that a bin $B$ is inactive only if $\hat\psi_B\leq\max\{\kappa_1\gamma^2(d+1)^3,\kappa_1'\gamma d^{3/2}\}/\sqrt{n}$
 which implies $\psi_B\leq 2\max\{\kappa_1\gamma^2(d+1)^3,\kappa_1'\gamma d^{3/2}\}/\sqrt{n}$ thanks to Lemma \ref{lem:lplr-update}, or that  $k_h=M$;
Eq.~(\ref{eq:proof-epsi-5}) holds by the Cauchy-Schwarz inequality.
Summing both sides of Eq.~(\ref{eq:proof-epsi-6}) over $B\in\mB_{h-1}$ and noting that there are $M^{h-1}$ bins on level $h-1$, we complete the proof of Lemma \ref{lem:lplr-ci}.

\subsection{Proof of Lemma \ref{thm:lplr}}

For the first statement, note that if a bin $B$ is inactive then $|\phi^\top(\hat\theta_B^{\tot})-\theta^*)|\leq\hat\eta_B(\phi)$ automatically holds from Lemma \ref{lem:lplr-ci}.
If bin $B$ remains active all the way until that last epoch, then $\theta_B^*=\theta^*$ because the path from the first layer to the last layer consists of $d$ orthonormal vectors $\hat u$
and therefore the cumulative projection is the identity matrix. This proves the first property of Lemma \ref{thm:lplr}.


Next, we focus on proving the second claim.
When all the high probability events hold with probability $1-2d\delta$, by lemma \ref{lem:lplr-ci} we know for every layer $h$ and the collection of all bins $B\in\mB_h$ that
\begin{align}
    &\sum_{B\in\mB_h}\psi_B\sqrt{\mathbb E_\mD[\hat\eta_B(\phi)^2|\phi\in B]} \leq 2.54 \left[\sum_{B'\in\mB_{h-1}}\psi_{B'}\sqrt{\mathbb E_\mD[\hat\eta_{B'}(\phi)^2|\phi\in B']}\right]\nonumber\\
    &\;\;\;\;+ \frac{12M^h(\kappa_3\gamma^2(d+1)^4+\kappa_2\gamma d^2+\max\{\kappa_1\gamma^2(d+1)^3,\kappa_1'\gamma d^{3/2}\})}{\sqrt{n}}.
\end{align}
Iterating over all $h=d,d-1,\cdots,1$ we obtain
\begin{align}
\sum_{B\in\mB_d}\psi_B\sqrt{\mathbb E_\mD[\hat\eta_B(\phi)^2|\phi\in B]}
\leq (2.54M)^d\times \frac{12(\kappa_3\gamma^2(d+1)^4+\kappa_2\gamma d^2+\max\{\kappa_1\gamma^2(d+1)^3,\kappa_1'\gamma d^{3/2}\})}{\sqrt{n}}.
\end{align}
Finally, note that the left-hand side of the above inequality upper bounds $\sum_{B\in\mB_d}\psi_B\mathbb E_{\mD}[\hat\eta_B(\phi)|\phi\in B]=\mathbb E_{\mD}[\Delta(\phi)]$
by Cauchy-Schwarz inequality. This completes the proof of Lemma \ref{thm:lplr}.

\end{document}